\documentclass[final,12pt]{colt2023} 

\usepackage{mathrsfs}

\title[LLN for Bayesian two-layer Neural Network trained with Variational Inference]{Law of Large Numbers for Bayesian two-layer Neural Network trained with Variational Inference}
\usepackage{times}

\coltauthor{%
 \Name{Arnaud Descours} \Email{arnaud.descours@uca.fr}\\
 \addr Laboratoire de Math\'ematiques Blaise Pascal, UMR 6620, CNRS, Universit\'e Clermont Auvergne, Aubi\`ere, France.
 \AND
 \Name{Tom Huix} \Email{tom.huix@polytechnique.edu}\\
 \addr Centre de Mathématiques Appliquées, UMR 7641, Ecole polytechnique,  France. 
 \AND
 \Name{Arnaud Guillin} \Email{arnaud.guillin@uca.fr}\\
 \addr Laboratoire de Math\'ematiques Blaise Pascal, UMR 6620, CNRS, Universit\'e Clermont Auvergne, Aubi\`ere, France.
 \AND
 \Name{Manon Michel} \Email{manon.michel@uca.fr}\\
 \addr Laboratoire de Math\'ematiques Blaise Pascal, UMR 6620, CNRS, Universit\'e Clermont Auvergne, Aubi\`ere, France.
 \AND
 \Name{{\'E}ric Moulines} \Email{eric.moulines@polytechnique.edu}\\
 \addr Centre de Mathématiques Appliquées, UMR 7641, Ecole polytechnique,  France. 
 \AND
 \Name{Boris Nectoux} \Email{boris.nectoux@uca.fr}\\
 \addr Laboratoire de Math\'ematiques Blaise Pascal, UMR 6620, CNRS, Universit\'e Clermont Auvergne, Aubi\`ere, France.
 }
\newcommand{\di}{\mathrm{d}} 
\newcommand{\btheta}{\boldsymbol\theta}
\newcommand{\indep}{\perp \!\!\! \perp}
\begin{document}

\maketitle

\begin{abstract}
  We provide a rigorous analysis of training by variational inference
  (VI) of Bayesian neural networks in the two-layer and infinite-width
  case. We consider a regression problem with a regularized evidence
  lower bound (ELBO) which is decomposed into the expected
  log-likelihood of the data and the Kullback-Leibler (KL) divergence
  between the a priori distribution and the variational
  posterior. With an appropriate weighting of the KL, we prove a law
  of large numbers for three different training schemes: (i) the
  idealized case with exact estimation of a multiple Gaussian integral
  from the reparametrization trick, (ii) a minibatch scheme using
  Monte Carlo sampling, commonly known as \emph{Bayes by Backprop},
  and (iii) a new and computationally cheaper algorithm which we
  introduce as \emph{Minimal VI}. An important result is that all
  methods converge to the same mean-field limit. Finally, we
  illustrate our results numerically and discuss the need for the
  derivation of a central limit theorem.
 \end{abstract}


\section{Introduction}
Deep Learning has led to a revolution in machine learning with
impressive successes.  However, some limitations of DL have been
identified and, despite, many attempts, our understanding of DL is
still limited. A long-standing problem is the assessment of predictive
uncertainty: DL tends to be overconfident in its predictions
\cite{abdar2021review}, which is a problem in applications such as
autonomous driving
\citep{mcallister2017concrete,michelmore2020uncertainty}, medical
diagnosis \citep{kendall2017uncertainties,filos2019systematic}, or
finance; cf
\cite{krzywinski2013importance,ghahramani2015probabilistic}. Therefore,
on the one hand, analytical efforts are being made to thoroughly
investigate the performance of DL; and on the other hand, many
approaches have been proposed to alleviate its shortcomings. The
Bayesian paradigm is an attractive way to tackle predictive
uncertainty, as it provides a framework for training uncertainty-aware
neural networks (NNs) (e.g.
\cite{ghahramani2015probabilistic,blundell2015weight,gal2016dropout}).

Thanks to a fully probabilistic approach, Bayesian Neural Networks
(BNN) combine the impressive neural-network expressivity with the
decision-theoretic approach of Bayesian inference, making them capable
of providing predictive uncertainty; see
\cite{blundell2015weight,michelmore2020uncertainty,mcallister2017concrete,filos2019systematic}.
However, Bayesian inference requires deriving the posterior
distribution of the NN weights. This posterior distribution is
typically not tractable. A classical approach is to sample the
posterior distribution using Markov chain Monte Carlo methods (such as
Hamilton-Monte-Carlo methods). There are however long-standing
difficulties, such as the proper choice of the prior and fine-tuning
of the sampler.  Such difficulties often become prohibitive in
large-dimensional cases,\citep{cobb2021scaling}. An alternative is to
use variational inference, which has a long history
\citep{Hinton93keepingneural,mackay1995probable,mackay1995ensemble}. Simpler
methods that do not require exact computation of integrals over the
variational posterior were then developed, e.g. first by
\cite{graves2011practical} thanks to some approximation and then by
\cite{blundell2015weight} with the \emph{Bayes by Backprop}
approach. In the latter, the posterior distribution is approximated by
a parametric distribution and a generalisation of the
reparametrization trick used by \cite{kingma2014} leads to an unbiased
estimator of the gradient of the ELBO; see also
\cite{gal2016dropout,louizos2017multiplicative,khan2018fast}. Despite
the successful application of this approach, little is known about the
overparameterized limit and appropriate weighting that must be assumed
to obtain a nontrivial Bayesian posterior, see
\cite{izmailov2021bayesian}. Recently, \cite{huix} outlined  the
importance of balancing in ELBO the integrated log-likelihood term and
the KL regularizer, to avoid both overfitting and dominance of the
prior. However, a suitable limiting theory has yet to be established,
as well as guarantees for the practical implementation of the
stochastic gradient descent (SGD) used to estimate the parameters of
the variational distribution.

Motivated by the need to provide a solid theoretical framework,
asymptotic analysis of NN has gained much interest recently.  The main focus
has been on the gradient descent algorithm and its variants
\citep{Vanden2,chizat2018global,mei2018mean,sirignano2020lln,descours2022law}. In
much of these works, a mean-field analysis is performed to
characterize the limiting nonlinear evolution of the weights of a
two-layer NN, allowing the derivation of a law of large numbers and a
central limit theorem for the empirical distribution of neuron
weights. A long-term
goal of these works is to demonstrate convergence toward a global
minimum of these limits for the mean field. Despite some progress in
this direction, this is still an open and highly challenging problem;
cf \cite{chizat2018global,chizat22,ccff}. Nevertheless, this
asymptotic analysis is also of interest in its own right, as we show
here in the case of variational inference for Bayesian neural
networks. Indeed, based on this asymptotic analysis, we develop an
efficient and new variant of the stochastic gradient descent (SGD)
algorithm for variational inference in BNN that computes only the
information necessary to recover the limit behavior.

Our goal, then, is to work at the intersection of analytical efforts
to gain theoretical guarantees and insights and of practical methods
for a workable variational inference procedure. By adapting the
framework developed by \cite{descours2022law}, we produce a rigorous
asymptotic analysis of BNN trained in a variational setting for a
regression task. {From the limit equation analysis, we first
  find that a proper regularisation of the Kullback-Leibler divergence
  term in relation with the integrated loss leads to their right
  asymptotic balance. Second, we prove the asymptotic equivalence of
  the idealized and Bayes-by-Backprop SGD schemes, as both preserve
  the same core contributions to the limit. Finally, we introduce a
  computationally more favourable scheme, directly stemming from the
  effective asymptotic contributions. This scheme is the true
  mean-field algorithmic approach, as only deriving from
  non-interacting terms.}

More specifically, our contributions are the following:
\begin{itemize}
\item We first focus on the idealized SGD algorithm, where the
  variational expectations of the derivative of the loss from the
  reparametrization trick of \cite{blundell2015weight} are computed
  exactly. More precisely, we prove that with the number of neurons
  $N\to +\infty$, the sequence of trajectories of the scaled empirical
  distributions of the parameters satisfies a law of large
  numbers. This is the purpose of Theorem \ref{thm.ideal}. The proof
  is completely new: it establishes directly the limit in the topology
  inherited by the Wasserstein distance bypassing the highly technical
  Sobolev space arguments used in \cite{descours2022law}.
\end{itemize}
The idealized SGD requires the computation of some integrals, which in
practice prevents a direct application of this algorithm. However, we
can prove its convergence to an explicit nonlinear process. These
integrals are usually obtained by a Monte Carlo approximation, leading to the
\emph{Bayes-by-Backprop} SGD, see \cite{blundell2015weight}.
\begin{itemize}
\item We show for the \emph{Bayes-by-Backprop} SGD (see Theorem
  \ref{thm.z1zN}) that the sequence of trajectories of the scaled
  empirical distributions of the parameters satisfies the same law of
  large numbers as that in Theorem \ref{thm.ideal}, which justifies
  such an approximation procedure. Note that each step of the
  algorithm involves the simulation of $O(N)$ Gaussian random
  variables, which can make the associated gradient evaluation
  prohibitively expensive.
\item A careful analysis of the structure of the limit equation
  \eqref{eq_limit} allows us to develop a new algorithm, called
  \emph{Minimal-VI } SGD, which at each step generates only two
  Gaussian random variables and for which we prove the same limiting
  behavior. The key idea here is to keep only those contributions which
  affect the asymptotic behavior and which can be understood as the
  mean-field approximation from the uncorrelated degrees of
  freedom. This is all the more interesting since
  we observe numerically that the number weights $N$ required to reach
  this asymptotic limit is quite small which makes this variant of
  immediate practical interest.
\item We numerically investigate the convergence of the three methods
  to the common limit behavior on a toy example. We observe that the
  mean-field method is effective for a small number of neurons
  ($N=300$). The differences between the methods are
  reflected in the variances.
\end{itemize}
The paper is organized as follows: Section~\ref{sec:notation}
introduces the variational inference in BNN, as well as the SGD
schemes commonly considered, namely the idealized and
\emph{Bayes-by-backprop} variants. Then, in Section~\ref{sec:ideal} we
establish our initial result, the LLN for the idealized SGD. In
Section~\ref{sec:BbB} we prove the LLN for the
\emph{Bayes-by-backprop} SGD and its variants. We show that both SGD
schemes have the same limit behavior. Based on an analysis of the
obtained limit equation, we present in Section \ref{sec:minimal} the
new \emph{minimal- VI}. Finally, in Section~\ref{sec:numerics} we
illustrate our findings using numerical experiments. The proofs of the
mean-field limits, which are original and quite technically demanding,
are gathered in the supplementary paper.

\paragraph{Related works.}
Law of Large Numbers (LLN) for mean-field interacting particle
systems, have attracted a lot of attentions; see for
example~\cite{hitsuda1986tightness,sznitman_topics_1991,
  fernandez1997hilbertian, jourdainAIHP,delarue, delmoral,
  kurtz2004stochastic} and references therein.  The use of mean-field
particle systems to analyse two-layer neural networks with random
initialization have been considered in \cite{mei2018mean, mei2}, which
establish a LLN on the empirical measure of the weights at fixed times
- we consider in this paper the trajectory convergence{, i.e. the whole empirical measure process (time indexed) converges uniformly w.r.t. Skorohod topology. It enables not only to use the limiting PDE, for example to study the convergence of the weights towards the infimum of the loss function (see \cite{chizat2018global} for preliminary results), but is is also crucial to establish the central limit theorem, see for example \cite{descours2022law}}. \cite{Vanden2} give conditions for global convergence of
GD for exact mean-square loss and online stochastic gradient descent
(SGD) with mini-batches increasing in size with the number of weights
$N$.  A LLN for the entire trajectory of the empirical measure is also
given in \cite{sirignano2020lln} for a standard SGD.
\cite{durmus-neural} establish the propagation of chaos for SGD with
different step size schemes. Compared to the existing literature
dealing with the SGD empirical risk minimization in two-layer neural
networks, \cite{descours2022law} provide the first rigorous proof of
the existence of the limit PDE, and in particular its uniqueness, in
the LLN.

We are interested here in deriving a LLN but for Variational Inference
(VI) of two-layer Bayesian Neural Networks (BNN), where we consider a
regularized version of the Evidence Lower Bound (ELBO).

\section{Variational inference in BNN: Notations and common SGD
  schemes}
\label{sec:notation}
\subsection{Variational inference and Evidence Lower Bound}
\textbf{Setting}. Let $\mathsf X$ and $\mathsf Y$ be subsets of $\mathbf R^n$ ($n\ge 1$) and $\mathbf R$ respectively.
For $N\ge1$ and  $\boldsymbol{w}=(w^1,\dots,w^N)\in(\mathbf R^d)^N$, let $f_{\boldsymbol{w}}^N: \mathsf X\to \mathbf R$ be the following two-layer neural network: for $x\in\mathsf X$,
\begin{equation*}
f_{\boldsymbol{w}}^N(x):=\frac 1N\sum_{i=1}^Ns(w^i,x)\in\mathbf R,
\end{equation*}
where $s:\mathbf R^d\times \mathsf X\to \mathbf R$ is the activation function.
We work in a Bayesian setting, in which we seek a distribution of the latent variable $\boldsymbol{w}$ which represents the weights of the neural network. The standard problem in Bayesian inference over complex models is that the posterior distribution is hard to sample. To tackle this problem,  we consider Variational Inference, in which we consider a family of distribution $\mathcal Q^N=\{ q_{\boldsymbol\theta}^N, \boldsymbol\theta\in \Xi^N\}$ (where $\Xi$ is some parameter space)  easy to sample. The objective is to find the best $q_{\boldsymbol\theta}^N\in\mathcal Q^N$, the one closest in KL divergence (denoted $\mathscr D_{{\rm KL}}$) to the exact posterior. Because we cannot compute the KL, we optimize the evidence lower bound (ELBO), which is equivalent to the KL up to an additive constant.

Denoting by $\mathfrak L: \mathbf R\times\mathbf R\to\mathbf R_+$ the negative log-likelihood (by an abuse of language, we call this quantity the \emph{loss}),  the  ELBO (see \cite{blei2017variational}) is defined, for $\btheta\in \Xi^N$, $(x,y)\in\mathsf X\times\mathsf Y$,  by
\begin{equation*}\label{d-elbo}
\mathrm{E}_{{\rm lbo}}(\boldsymbol\theta,x,y) :=- \int_{(\mathbf R^d)^N}\mathfrak L(y,f_{\boldsymbol{w}}^N(x))q_{\boldsymbol\theta}^N(\boldsymbol{w})\di \boldsymbol{w} - \mathscr D_{{\rm KL}}(q_{\btheta}^N|P_0^N),
\end{equation*}
where $P_0^N$ is some prior on the weights of the NN.  The ELBO is
decomposed into two terms: one corresponding to the Kullback-Leibler
(KL) divergence between the variational density and the prior and the
other to a marginal likelihood term. It was empirically found that the
maximization of the ELBO function is prone to yield very poor
inferences \citep{coker2021wide}. It is argued in \cite{coker2021wide} and
\cite{huix} that optimizing the ELBO leads as $N \to \infty$ to the
collapse of the variational posterior to the prior. \cite{huix}
proposed to consider a regularized version of the ELBO, which consists
in multiplying the KL term by a parameter which is scaled by the
inverse of the number of neurons:
\begin{equation}\label{d-elbo_w}
\mathrm{E}_{{\rm lbo}}^N(\boldsymbol\theta,x,y) :=- \int_{(\mathbf R^d)^N}\mathfrak L(y,f_{\boldsymbol{w}}^N(x))q_{\boldsymbol\theta}^N(\boldsymbol{w})\di \boldsymbol{w} -\frac 1N \mathscr D_{{\rm KL}}(q_{\btheta}^N|P_0^N),
\end{equation}
A first objective of this paper is to show
 that the proposed regularization leads to a stable asymptotic behavior
and the effect of both the integrated loss and Kullback-Leibler terms on the
limiting behavior are balanced in the limit $N \to \infty$.
The maximization of $\mathrm{E}_{{\rm lbo}}^N$  is carried out using SGD.

The variational family $\mathcal Q^N$ we consider is a Gaussian family of distributions. More precisely,   we assume that for any $\btheta=(\theta^1,\dots,\theta^N)\in\Xi^N$, the variational distribution $q_{\btheta}^N$ factorizes over the neurons: for all $\boldsymbol{w}=(w^1,\dots,w^N)\in(\mathbf R^d)^N$, $q_{\btheta}^N(\boldsymbol{w})=\prod_{i=1}^Nq^1_{\theta^i}(w^i)$, where
$\theta=(m,\rho)\in\Xi:=\mathbf R^d\times\mathbf R$ and  $q^1_\theta$ is the probability density function (pdf) of $\mathcal N(m,g(\rho)^2 I_d)$, with $g(\rho)=\log(1+e^{\rho}), \ \rho \in \mathbf R.$

In the following,  we simply  write $\mathbf R^{d+1}$ for $\mathbf R^d\times\mathbf R$.
In addition, following the reparameterisation trick of \cite{blundell2015weight},   $q^1_\theta(w) \di w$  is the pushforward of a reference probability measure with density $\gamma$ by $\Psi_\theta$ (see more precisely Assumption {\rm \textbf{A1}}).
In practice, $\gamma$ is the pdf of $\mathcal N(0,I_d)$ and $\Psi_\theta(z)=m+g(\rho)z$. With these notations, \eqref{d-elbo_w} writes
\begin{align}
\nonumber
\mathrm{E}_{{\rm lbo}}^N(\boldsymbol\theta,x,y) &=- \int_{(\mathbf R^d)^N}\!\!\mathfrak L\Big(y,\frac 1N\sum_{i=1}^Ns(\Psi_{\theta^i}(z^i),x)\Big) \gamma(z^1)\dots\gamma(z^N)\di z_1\dots\di z_N   -\frac 1N \mathscr D_{{\rm KL}}(q_{\btheta}^N|P_0^N).
\end{align}
\noindent
\textbf{Loss function and prior distribution}.
In this work, we focus on the regression problem, i.e.
$\mathfrak L$ is the Mean Square Loss:   for $y_1,y_2\in\mathbf R$,   $\mathfrak L(y_1,y_2)=\frac 12|y_1-y_2|^2$.
 We also introduce the function $\phi:(\theta,z,x)\in \mathbf R^{d+1} \times\mathbf R^d\times\mathsf X\mapsto s(\Psi_\theta(z),x).$ On the other hand,  we  assume that the  prior distribution $P_0^N$ write, for all $\boldsymbol{w}\in(\mathbf R^d)^N$,
$P_0^N(\boldsymbol{w})=\prod_{i=1}^NP_0^1(w^i)$,
 where $P_0^1:\mathbf R^d\to\mathbf R_+$  is the pdf of $\mathcal N(m_0,\sigma^2_0I_d)$, and $\sigma_0>0$. Therefore $\mathscr D_{{\rm KL}}(q_{\btheta}^N|P_0^N)=\sum_{i=1}^N\mathscr D_{{\rm KL}}(q_{\theta^i}|P_0^1)$ and, for $\theta=(m,\rho)\in \mathbf R^{d+1}$,
\begin{align*}
\mathscr D_{{\rm KL}}(q_\theta^1|P_0^1)=\int_{\mathbf R^d} q^1_\theta(x) \log(q^1_\theta(x)/P_0^1(x))\di x=\frac{\|m-m_0\|_2^2}{2\sigma_0^2}+\frac d2\Big(\frac{g(\rho)^2}{\sigma_0^2}-1\Big)+\frac d2\log\Big(\frac{\sigma_0^2}{g(\rho)^2}\Big).
\end{align*}
Note  that $\mathscr D_{{\rm KL}}$ has at most a quadratic growth in $m$ and $\rho$.

{Note that we assume here a Gaussian prior to get an explicit expression of the Kullback-Leibler divergence. Most arguments extend to sufficiently regular densities and are essentially the same for exponential families, using conjugate families for the variational approximation.}
\subsection{Common SGD schemes in backpropagation in a variational setting}
 {\noindent \bf Idealized SGD.} Let $(\Omega, \mathcal F,\mathbf P)$ be a probability space. Consider a data set  $\{(x_k,y_k)\}_{k\ge 0}$    i.i.d.    w.r.t. $\pi\in\mathcal{P}(\mathsf X\times\mathsf Y)$, the space of probability measures over $\mathsf X\times\mathsf Y$. For $N\ge1$ and given a  learning rate $\eta>0$,  the maximization of $\theta\in \mathbf R^{d+1}\mapsto \mathrm{E}_{{\rm lbo}}^N(\boldsymbol\theta,x,y)$ with a SGD algorithm writes as follows:
  for $k\ge 0$ and $i\in\{1,\dots,N\}$,
\begin{equation}\begin{cases}\label{eq.sgd}
 &\boldsymbol\theta_{k+1}=\boldsymbol\theta_k+ \eta \nabla_{\boldsymbol\theta}\mathrm{E}_{{\rm lbo}}^N(\boldsymbol\theta_k,x_k,y_k) \\
 &\boldsymbol\theta_0 \sim \mu_0^{\otimes N},
\end{cases}\end{equation}
 where $\mu_0\in \mathcal P(\mathbf R^{d+1})$ and $\boldsymbol\theta_k=(\theta^1_k,\ldots, \theta^N_k)$.
We now compute $\nabla_{\boldsymbol\theta}\mathrm{E}_{{\rm lbo}}^N(\boldsymbol\theta,x,y)$.

First, under  regularity assumptions on the function $\phi$ (which will be formulated later, see \textbf{A1} and \textbf{A3} below) and by assumption on $\mathfrak L$, we have for all $i\in\{1,\dots,N\}$ and all $(x,y)\in\mathsf X\times\mathsf Y$,
\begin{align}
\nonumber
&\int_{(\mathbf R^d)^N} \nabla_{\theta^i}\mathfrak L\Big(y,\frac 1N\sum_{j=1}^N\phi(\theta^j,z^j,x)\Big)\gamma(z^1)\dots\gamma(z^N)\di z^1\dots\di z^N\nonumber\\
&= -\frac{1}{N^2}\sum_{j=1}^N\int_{(\mathbf R^d)^N}(y-\phi(\theta^j,z^j,x))\nabla_{\theta}\phi(\theta^i,z^i,x)\gamma(z^1)\dots\gamma(z^N)\di z^1\dots\di z^N\label{grad_int0}\\
\nonumber
&=-\frac{1}{N^2}\Big[\sum_{j=1,j\neq i}^N(y-\langle\phi(\theta^j,\cdot,x),\gamma\rangle)\langle\nabla_{\theta}\phi(\theta^i,\cdot,x),\gamma\rangle + \langle(y-\phi(\theta^i,\cdot,x))\nabla_{\theta}\phi(\theta^i,\cdot,x),\gamma\rangle\Big],
\end{align}
where we have used the notation $\langle U,\nu \rangle=\int_{\mathbf R^q}U(z)\nu(\di z) $ for any integrable   function  $U:\mathbf R^q\to \mathbf R$ w.r.t. a measure~$\nu$ (with a slight abuse of notation,  we  denote by $\gamma$ the measure $\gamma(z)\di z$). Second, for $\theta \in \mathbf R^{d+1}$, we have
\begin{align}\label{eq.kl_1}
\nabla_{\theta}\mathscr D_{{\rm KL}}(q_{\theta}^1|P_0^1)=
\begin{pmatrix}
  \nabla_{m}\mathscr D_{{\rm KL}}(q_{\theta}^1|P_0^1)    \\
    \partial_{\rho}\mathscr D_{{\rm KL}}(q_{\theta}^1|P_0^1)
\end{pmatrix}
=
\begin{pmatrix}
  \frac{1}{\sigma_0^2}(m-m_0)    \\
      \frac{d}{\sigma_0^2}g'(\rho)g(\rho)-d\frac{g'(\rho)}{g(\rho)}
\end{pmatrix}.
\end{align}
In conclusion,   the SGD \eqref{eq.sgd} writes: for $k\ge 0$ and $i\in\{1,\dots,N\}$,
\begin{equation}\begin{cases}\label{eq.algo-ideal}
 &\theta_{k+1}^i=\theta_{k}^i\displaystyle-\frac{\eta}{N^2}\sum_{j=1,j\neq i}^N\Big(\langle\phi(\theta_k^j,\cdot,x_k),\gamma\rangle-y_k\Big)\langle\nabla_\theta\phi(\theta_k^i,\cdot,x_k),\gamma\rangle \\
& \qquad \quad \displaystyle-\frac{\eta}{N^2}\Big\langle(\phi(\theta_k^i,\cdot,x_k)-y_k)\nabla_\theta\phi(\theta_k^i,\cdot,x_k),\gamma\Big\rangle-\frac{\eta}{N}\nabla_{\theta}\mathscr D_{{\rm KL}}(q_{\theta^i_k}^1|P_0^1)\\
 &\theta_{0}^i \sim \mu_0.
\end{cases}\end{equation}
We shall call this algorithm \emph{idealised} SGD because it contains an intractable term given by the integral w.r.t. $\gamma$. This has motivated the development of methods where this integral is replaced by an unbiased Monte Carlo estimator (see \cite{blundell2015weight}) as detailed below.\\

\noindent\textbf{\emph{Bayes-by-Backprop} SGD}. The second  SGD algorithm we study
 is based on an approximation, for $i\in\{1,\dots,N\}$, of  $\int_{(\mathbf R^d)^N}(y-\phi(\theta^j,z^j,x))\nabla_{\theta}\phi(\theta^i,z^i,x)\gamma(z^1)\dots\gamma(z^N)\di z^1\dots\di z^N$  (see \eqref{grad_int0})
by
\begin{equation}\label{eq.sum-a}
  \frac 1B\sum_{\ell=1}^B\big (y-\phi(\theta^j, \mathsf Z^{j,\ell},x)\big )\nabla_\theta\phi(\theta^i,\mathsf  Z^{i,\ell},x)
\end{equation}
where $B\in \mathbf N^*$ is a fixed integer and $(\mathsf Z^{q,\ell}, q\in \{i,j\},  1\le \ell\le B)$ is a i.i.d  finite sequence of random variables distributed according to $\gamma(z)\di z$.
In this case, for $N\ge 1$, given a dataset $(x_k,y_k)_{k\ge0}$,  the maximization of $\theta\in \mathbf R^{d+1}\mapsto \mathrm{E}_{{\rm lbo}}^N(\boldsymbol\theta,x,y)$ with a SGD algorithm   is the following: for $k\ge 0$ and $i\in\{1,\dots,N\}$,
\begin{equation}\begin{cases}\label{eq.algo-batch}
&\displaystyle \theta_{k+1}^i=\theta_k^i\displaystyle -\frac{\eta}{N^2B}\sum_{j=1}^N\sum_{\ell=1}^B\big (\phi(\theta_k^j,\mathsf Z^{j,\ell}_{k},x_k)-y_k\big )\nabla_\theta\phi(\theta_k^i,\mathsf Z^{i,\ell}_k,x_k)
-\frac{\eta}{N}\nabla_\theta \mathscr D_{{\rm KL}}(q_{\theta^i_k}^1|P_0^1)\\
 &\theta_{0}^i=(m_{0}^i,\rho_{0}^i)\sim \mu_0,
\end{cases}\end{equation}
where $\eta>0$ and  $(\mathsf Z^{j,\ell}_k, 1\le j\le N, 1\leq\ell\leq B, k\ge 0)$ is a i.i.d  sequence of random variables distributed according to $\gamma$.


\section{Law of large numbers for the idealized SGD}
\label{sec:ideal}
\textbf{Assumptions and notations}.  When $E$ is a metric space and $\mathscr I= \mathbf R_+$ or $\mathscr I=[0,T]$ ($T\ge 0$), we  denote by $\mathcal D(\mathscr  I,E)$ the Skorohod space of c\`adl\`ag functions on $\mathscr I$ taking values in  $E$ and $\mathcal C(\mathscr  I,E)$ the space of continuous functions on $\mathscr I$ taking values in  $E$.
The evolution of the parameters $(\{\theta_k^i, i=1,\ldots,N\})_{k\ge 1}$ defined by \eqref{eq.algo-ideal} is tracked through their empirical distribution $\nu_k^N$ (for $k\geq 0$) and its scaled version $\mu_t^N$ (for $t\in\mathbf{R}_+$), which are defined as follows:
\begin{equation}\label{empirical_distrib1}
\nu_k^N:=\frac 1N\sum_{i=1}^N\delta_{\theta_k^i} \ \ \text{and} \ \ \mu_t^N:=\nu_{\lfloor Nt\rfloor}^N,  \text{ where the  $\theta^i_k$'s are defined  \eqref{eq.algo-ideal}}.
\end{equation}
Fix $T>0$.
For all $N\ge1$, $\mu^N:=\{\mu_t^N, t\in[0,T]\}$ is a random element of $\mathcal D([0,T],\mathcal P(\mathbf R^{d+1}))$, where $\mathcal P(\mathbf R^{d+1})$ is endowed with the weak convergence topology. For $N\ge1$ and $k\ge1$, we introduce the following $\boldsymbol\sigma$-algebras:
\begin{align}\label{eq.Fk1}
\mathcal F_0^N=\boldsymbol\sigma(\theta_{0}^i, 1\le i\le N) \ \ \text{and} \ \ \mathcal F_k^N=\boldsymbol\sigma(\theta_{0}^i,   (x_q,y_q),1\le i\le N, 0\le q\le k-1).
\end{align}
Recall  $q_\theta^1:\mathbf R^d\to\mathbf R_+$ be the pdf of $\mathcal N(m,g(\rho)^2I_d)$ ($\theta=(m,\rho)\in\mathbf R^{d+1}$).
In this work, we assume the following.
\begin{enumerate}
\item[\textbf{A1}.]
There exists a pdf $\gamma:\mathbf R^d\to\mathbf R_+$ such that for all $\theta\in \mathbf R^{d+1}$, $q^1_\theta\di x=\Psi_\theta\#\gamma\di x$, where $\{\Psi_\theta, \theta\in\mathbf R^{d+1}\}$ is a family of $\mathcal C^1$-diffeomorphisms over $\mathbf R^d$ such that for all $z\in\mathbf R^d$, $\theta\in\mathbf R^{d+1}\mapsto \Psi_\theta(z)$ is of class $\mathcal C^\infty$.
Finally, there exists $\mathfrak b:\mathbf R^d\to\mathbf R_+$   such that for all multi-index $\alpha \in \mathbf N^{d+1}$ with $|\alpha|\ge 1$, there exists $C_\alpha>0$, for all $z\in\mathbf R^d$ and $  \theta=(\theta_1,\ldots,\theta_{d+1})\in \mathbf R^{d+1}$,
\begin{equation}\label{jac_T_bounded}
\big| \partial_{\alpha}\Psi_\theta(z)\big|  \leq C_{\alpha} \mathfrak b(z) \ \ \text{ with for all } q\ge 1, \  \langle  \mathfrak b^q, \gamma\rangle <+\infty,
\end{equation}
 where $\partial_\alpha= \partial_{\theta_1}^{\alpha_1}\ldots \partial_{\theta_{d+1}}^{\alpha_{d+1}}$ and $\partial_{\theta_j}^{\alpha_j}$  is the partial derivatives of order $\alpha_j$ w.r.t. to $\theta_j$.

\item[\textbf{A2.}]
The sequence $\{(x_k,y_k)\}_{k\ge 0}$  is  i.i.d.   w.r.t. $  \pi\in\mathcal{P}(\mathsf X\times\mathsf Y)$.
The set $\mathsf X\times\mathsf Y\subset \mathbf R^d\times \mathbf R$ is compact. For all $k\ge0$, $(x_k,y_k)\indep \mathcal F_{k}^N$, where $\mathcal F_k^N$ is defined in \eqref{eq.Fk1}.
\item[\textbf{A3.}]
 The activation function $s:\mathbf R^d\times \mathsf X\to\mathbf R$ belongs to  $\mathcal C^\infty_b(\mathbf R^d\times \mathsf X)$ (the space of smooth functions over $\mathbf R^d\times\mathsf X$ whose derivatives of all order are bounded).
\item[\textbf{A4.}]
 The initial parameters $(\theta_{0}^i)_{i=1}^N$ are i.i.d. w.r.t. $\mu_0\in \mathcal P(\mathbf R^{d+1})$ which has compact support.
\end{enumerate}

Note that {\rm \textbf{A1}} is satisfied  when  $\gamma$ is the pdf of $\mathcal N(0,I_d)$ and $\Psi_\theta(z)=m+g(\rho)z$, with $\mathfrak b(z)=1+|z|$.
With these assumptions, for every fixed $T>0$,  the sequence  $(\{\theta_k^i, i=1,\ldots,N\})_{k=0, \ldots, \lfloor NT \rfloor}$ defined by \eqref{eq.algo-ideal} is a.s.  bounded:

\begin{lemma}[Uniform bound on the parameters]
\label{lem:unif_bound_param}
  Assume {\rm \textbf{A1}}$\to${\rm \textbf{A4}}. Then,
  there exists $C>0$ such that a.s. for all $T>0$, $N\ge 1$,
  $i\in \{1,\dots, N\}$, and $0\leq k\leq \lfloor NT\rfloor$,
  $|\theta_k^i|\leq Ce^{[ C(2+T)]T}$.
\end{lemma}
\noindent
Lemma \ref{lem:unif_bound_param} implies that a.s.  for all $T>0$ and $N\ge 1$, $\mu^N \in \mathcal D([0,T],\mathcal P(\Theta_T))$, where
$$\Theta_T=\{\theta \in \mathbf R^{d+1}, |\theta|\le Ce^{[ C(2+T)]T}\}.$$


 \noindent
\textbf{Law of large numbers for $(\mu^N)_{N\ge1}$  defined in \eqref{empirical_distrib1}}.  The first main result of this work is the following.


\begin{theorem}\label{thm.ideal}
Assume  {\rm \textbf{A1}}$\to${\rm \textbf{A4}}. Let $T>0$. Then, the sequence   $(\mu^N)_{N\ge1}\subset \mathcal D([0,T],\mathcal P(\Theta_T))$  defined in \eqref{empirical_distrib1} converges in probability to the unique deterministic solution $\bar\mu\in \mathcal C([0,T],\mathcal P(\Theta_T))$ to the following measure-valued evolution equation: $\forall f\in\mathcal C^\infty(\Theta_T) \text{ and } \forall t\in [0,T],$
\begin{align}\label{eq_limit}
\langle f,\bar\mu_t\rangle-\langle f,\mu_0\rangle&=- \eta\int_{0}^t\int_{\mathsf X\times\mathsf Y}\big \langle\phi(\cdot,\cdot,x)-y,\bar\mu_s\otimes\gamma\big \rangle\big \langle\nabla_\theta f\cdot\nabla_\theta\phi( \cdot ,\cdot,x),\bar\mu_s\otimes\gamma\big \rangle  \pi(\di x,\di y)\di s\nonumber\\
&\quad- \eta\int_0^t\big \langle\nabla_\theta f\cdot \nabla_\theta \mathscr D_{{\rm KL}}(q_{\,_\cdot }^1|P_0^1),\bar\mu_s\big \rangle\di s.
\end{align}
\end{theorem}
The proof of Theorem \ref{thm.ideal} is given in Appendix
\ref{sec.proof1}. We stress here the most important steps and used
techniques. In a first step, we derive an identity satisfied by
$(\mu^N)_{N\ge 1}$, namely the pre-limit
equation~\eqref{eq.pre_limit}; see Sec.~\ref{sec.pre-eq-1}.  Then we
show in Sec.~\ref{sec.RC-1} that $(\mu^N)_{N\ge 1}$ is relatively
compact in $\mathcal D([0,T],\mathcal P(\Theta_T))$.
To do so, we check that the sequence $(\mu^N)_{N\ge 1}$ satisfies all the required assumptions of \cite[Theorem 3.1]{jakubowski1986skorokhod} when $E= \mathcal P(\Theta_T)$ there.
In
Sec.~\ref{sec.LP-1} we prove that every limit point of
$(\mu^N)_{N\ge 1}$ satisfies the limit equation~\eqref{eq_limit}. Then, in Section \ref{sec.U-1},
we prove that there is a unique solution of the measure-valued equation~\eqref{eq_limit}.
To prove the uniqueness of the solution of~\eqref{eq_limit},
we use techniques  developed in~\cite{piccoli2015control} which are based on
a  representation formula for solution to measure-valued equations~\cite[Theorem 5.34]{villani2021topics} together with estimates in  Wasserstein distances between two solutions of~\eqref{eq_limit} derived in \cite{piccoli2016properties}.
In  Section \ref{sec.U-1}, we also conclude the
proof of Theorem~\ref{thm.ideal}.  Compared
to~\cite[Theorem~1]{descours2022law}, the fact that
$(\{\theta_k^i, i=1,\ldots,N\})_{k=0, \ldots, \lfloor NT \rfloor}$
defined by \eqref{eq.algo-ideal} are a.s.  bounded allows to use
different and more straightforward arguments to prove (i) the relative
compactness in $\mathcal D([0,T],\mathcal P(\Theta_T))$ of
$(\mu^N)_{N\ge1}$ (defined in~\eqref{empirical_distrib1}) (ii) the
continuity property of the operator
$\mathsf m\mapsto \boldsymbol{\Lambda}_t[f](\mathsf m)$ defined in
\eqref{d-lambda} w.r.t. the topology of
$\mathcal D([0,T],\mathcal P(\Theta_T))$ and (iii) $(\mu^N)_{N\ge 1}$
has limit points in $\mathcal C([0,T],\mathcal P(\Theta_T))$. Step
(ii) is necessary in order to pass to the limit $N\to +\infty$ in the
pre-limit equation and Step (iii) is crucial since  we prove that  there is at most
one solution of~\eqref{eq_limit} in
$\mathcal C([0,T],\mathcal P(\Theta_T))$.  It is worthwhile to
emphasize that, as $N \to \infty$, the effects of the integrated loss
and of the KL terms are balanced, as conjectured in \cite{huix}.




{To avoid further technicalities, we have chosen what may seem restrictive assumptions on the data or the activation function. Note however that it readily extends to unbounded set  $\mathsf X$, and also unbounded $\mathsf Y$ assuming that $\pi$ as polynomial moments of sufficiently high order. Also, RELU (or more easily leaky RELU) may be considered by using weak derivatives (to consider the singularity at 0), and a priori moment bounds on the weights.}

\section{LLN for the \emph{Bayes-by-Backprop} SGD}
\label{sec:BbB}
The sequence $\{\theta_k^i, i\in \{1,\ldots N\}\}_{k=0, \ldots, \lfloor NT \rfloor}$ defined recursively by the algorithm~\eqref{eq.algo-batch} is in general not bounded, since $\nabla_\theta\phi(\theta ,\mathsf Z, x)$ is not necessarily bounded if $\mathsf Z\sim \gamma(s)\di z$. Therefore, we cannot expect Lemma \ref{lem:unif_bound_param} to hold for $\{\theta_k^i, i\in \{1,\ldots N\}\}_{k=0, \ldots, \lfloor NT \rfloor}$ set by~\eqref{eq.algo-batch}. Thus, the sequence $\{\theta_k^i, i\in \{1,\ldots N\}\}_{k=0, \ldots, \lfloor NT \rfloor}$ is considered on the whole space~$\mathbf R^{d+1}$.\\
\textbf{Wasserstein spaces and results}.
For $N\ge1$,  and   $k\ge 1$,  we set
\begin{equation}\label{eq.Fk2}
\mathcal F_k^N=\boldsymbol\sigma \Big (\theta_{0}^i ,   \mathsf Z^{j,\ell}_q,(x_q,y_q),  1\leq i,j\leq N, 1\leq\ell\leq B, 0\le q\le k-1\big \} \Big ).
\end{equation}
In addition to  {\rm \textbf{A1}}$\to${\rm \textbf{A4}}  (where in {\rm \textbf{A2}}, when $k\ge 1$, $\mathcal F_k^N$ is now the one defined in \eqref{eq.Fk2}),
we assume:
 \begin{enumerate}
\item[\textbf{A5}.]   The sequences $(\mathsf Z^{j,\ell}_k,1\leq j\leq N, 1\leq\ell\leq B, k\ge 0)$ and $((x_k,y_k), k\ge 0)$ are independent. In addition, for $k\ge 0$, $\big((x_k,y_k),\mathsf Z^{j,\ell}_k, 1\leq j\leq N, 1\leq\ell\leq B\big)\indep  \mathcal F_k^N$.
 \end{enumerate}

\noindent
Note that the last statement of \textbf{A5} implies the last statement of {\rm \textbf{A2}}.
We introduce the scaled empirical distribution of the parameters of the algorithm \eqref{eq.algo-batch}, i.e. for  $k\ge 0$ and $t\ge 0$:
\begin{equation}\label{empirical_distrib2}
\nu_k^N:=\frac 1N\sum_{i=1}^N\delta_{\theta_k^i} \ \ \text{and} \ \ \mu_t^N:=\nu_{\lfloor Nt\rfloor}^N, \text{ where the  $\theta^i_k$'s are defined  \eqref{eq.algo-batch}}.
\end{equation}
One can no longer rely on the existence of a compact subset $\Theta_T\subset \mathbf R^{d+1}$ such that a.s. $(\mu^N)_{N\ge1}\subset\mathcal D([0,T], \mathcal P(\Theta_T))$, where $\mu^N=\{t\ge 0\mapsto \mu_t^N\}$ is defined in \eqref{empirical_distrib2}. For this reason, we will work in Wasserstein spaces $\mathcal P_q(\mathbf R^{d+1})$, $q\ge 0$, which, we recall, are defined by
\begin{equation}\label{eq.Wq}
\mathcal P_q(\mathbf R^{d+1})=\Big\{ \nu  \in  \mathcal P(\mathbf R^{d+1}), \int_{\mathbf R^{d+1}} |\theta|^q \nu (\di \theta)<+\infty\Big\}.
\end{equation}
These spaces are endowed with the Wasserstein metric $\mathsf W_q$, see e.g.~\cite[Chapter 5]{santambrogio2015optimal} for more materials on Wasserstein spaces. For all $q\ge 0$,  $(\mu^N)_{N\ge1}\subset \mathcal D(\mathbf R_+,\mathcal P_q(\mathbf R^{d+1}))$.
The second main results   of this work is a LLN for   $(\mu^N)_{N\ge1}$ defined in \eqref{empirical_distrib2}.
\begin{theorem}\label{thm.z1zN}
Assume  {\rm \textbf{A1}}$\to${\rm \textbf{A5}}. Let $\gamma_0> 1+ \frac{d+1}{2}$.  Then, the sequence $(\mu^N)_{N\ge1}$ defined in~\eqref{empirical_distrib2}   converges in probability in $\mathcal D(\mathbf R_+,\mathcal P_{\gamma_0}(\mathbf R^{d+1}))$ to a deterministic element $\bar\mu\in \mathcal D(\mathbf R_+,\mathcal P_{\gamma_0}(\mathbf R^{d+1}))$, where  $\bar\mu\in \mathcal C(\mathbf R_+,\mathcal P_{1}(\mathbf R^{d+1}))$ is the unique solution in $\mathcal C(\mathbf R_+,\mathcal P_{1}(\mathbf R^{d+1}))$ to the following measure-valued evolution equation:$ \forall f\in \mathcal C^\infty_b(\mathbf R^{d+1}) \text{ and } \forall t\in \mathbf R_+$,
\begin{align}
\label{eq.P2}
\langle f,\bar\mu_t\rangle-\langle f,\mu_0\rangle&=- \eta\int_{0}^t\int_{\mathsf X\times\mathsf Y}\big \langle\phi(\cdot,\cdot,x)-y,\bar\mu_s\otimes\gamma\big \rangle\big \langle\nabla_\theta f\cdot\nabla_\theta\phi( \cdot ,\cdot,x),\bar\mu_s\otimes\gamma\big \rangle  \pi(\di x,\di y)\di s\nonumber\\
&\quad- \eta\int_0^t\big \langle\nabla_\theta f\cdot \nabla_\theta \mathscr D_{{\rm KL}}(q_{\,_\cdot }^1|P_0^1),\bar\mu_s\big \rangle\di s.
\end{align}
\end{theorem}
Theorem \ref{thm.z1zN} is proved in the appendix \ref{sec.proof2}.
Since~$\{\theta_k^i, i\in \{1,\ldots N\}\}_{k=0, \ldots, \lfloor NT \rfloor}$ defined by \eqref{eq.algo-batch} is not bounded in general, we work in the space $\mathcal D(\mathbf R_+, \mathcal P_{\gamma_0}(\mathbf R^{d+1}))$. The proof of Theorem \ref{thm.z1zN} is more involved than that of Theorem~\ref{thm.ideal}, and generalizes the latter to the case where the parameters of the SGD algorithm are unbounded.
We prove that $(\mu^N)_{N\ge1}$ (defined in \eqref{empirical_distrib2}) is relatively compact in $\mathcal D(\mathbf R_+, \mathcal P_{\gamma_0}(\mathbf R^{d+1}))$. To this end we now use \cite[Theorem 4.6]{jakubowski1986skorokhod}. The compact containment, which is the purpose of Lemma \ref{lem_cc_mu^N},  is not straightforward since  $\mathcal P_{\gamma_0}(\mathbf R^{d+1})$ is not compact contrary to  Theorem \ref{thm.ideal}  where we used the  compactness of  $\mathcal P(\Theta_T)$. More precisely, the compact containment  of $(\mu^N)_{N\ge 1}$ relies on a characterization of the compact subsets of  $\mathcal P_{\gamma_0}(\mathbf R^{d+1})$ (see Proposition \ref{prop_rc_inP}) and moment estimates on  $\{\theta_k^i, i\in \{1,\ldots N\}\}_{k=0, \ldots, \lfloor NT \rfloor}$ (see Lemma~\ref{lem:moment_param}).
We also mention that contrary to what is done in the proof of Theorem \ref{thm.ideal}, we do not show that every limit point  of $(\mu^N)_{N\ge1}$ in $\mathcal D(\mathbf R_+, \mathcal P_{\gamma_0}(\mathbf R^{d+1}))$ is continuous in time but we still manage to  prove that they all satisfy  \eqref{eq.P2}.  Then, using the duality formula for the $\mathsf W_1$-distance   together with rough estimates on the jumps of $t\mapsto \langle f, \mu_t^N\rangle$ (for $f$ uniformly Lipschitz over $\mathbf R^{d+1}$), we then show that    every limit point  of $(\mu^N)_{N\ge1}$ in $\mathcal D(\mathbf R_+, \mathcal P_{\gamma_0}(\mathbf R^{d+1}))$ belongs a.s. to $\mathcal C(\mathbf R_+, \mathcal P_{1}(\mathbf R^{d+1}))$. Again this is important since we have uniqueness of \eqref{eq.P2} in $\mathcal C(\mathbf R_+, \mathcal P_{1}(\mathbf R^{d+1}))$.

We conclude this section with the following important uniqueness result.
\begin{proposition}\label{pr.u}
Under the assumptions of Theorems \ref{thm.ideal} and \ref{thm.z1zN}, the solution to \eqref{eq_limit} is independent of $T$ and is equal to the solution to \eqref{eq.P2}.
\end{proposition}
This uniqueness result states that both idealized and
\emph{Bayes-by-backprop} SGD have the same limiting behavior. It is also noteworthy that the mini-batch $B$ is held fixed $B$. The effect of batch size can be seen at the level of the central limit theorem, which we leave for future work.

\section{The \emph{Minimal-VI} SGD algorithm}\label{sec:minimal}

The idea behing the \emph{Bayes-by-Backprop} SGD stems from the fact
that there are integrals wrt $\gamma$ in the loss function that cannot
be computed in practice and it is quite natural up to a
reparameterization trick, to replace these integrals by a Monte Carlo
approximation (with i.i.d. gaussian random variables).  To devise a
new cheaper algorithm based on the only terms impacting the asymptotic
limit, we directly analyse the limit equation \eqref{eq_limit} and remark that it can be rewritten as,
$\forall f\in\mathcal C^\infty(\Theta_T) \text{ and } \forall t\in
[0,T],$
\begin{align*}
&\langle f,\bar\mu_t\rangle-\langle f,\mu_0\rangle\\
&=- \eta\int_{0}^t\int_{\mathsf X\times\mathsf Y\times (\mathbf{R}^d)^2}\big \langle\phi(\cdot,z_1,x)-y,\bar\mu_s\big \rangle\big \langle\nabla_\theta f\cdot\nabla_\theta\phi( \cdot ,z_2,x),\bar\mu_s\big \rangle  \gamma^{\otimes 2}(\di z_1\di z_2)\pi(\di x,\di y)\di s\\
&\quad- \eta\int_0^t\big \langle\nabla_\theta f\cdot \nabla_\theta \mathscr D_{{\rm KL}}(q_{\,_\cdot }^1|P_0^1),\bar\mu_s\big \rangle\di s.
\end{align*}
Thus, the integration over $\gamma^{\otimes 2}$ can be considered as that over $\pi$, i.e., we can consider them as two more data variables that only need to be sampled at each new step. In this case, the SGD ~\eqref{eq.algo-batch}  becomes: for $k\ge 0$ and $i\in\{1,\dots,N\}$,
\begin{equation}\begin{cases}\label{eq.algo-z1z2}
&\displaystyle \theta_{k+1}^i=\theta_k^i\displaystyle -\frac{\eta}{N^2}\sum_{j=1}^N \big (\phi(\theta_k^j,\mathsf Z^{1}_{k},x_k)-y_k\big )\nabla_\theta\phi(\theta_k^i,\mathsf Z^{2}_k,x_k)
-\frac{\eta}{N}\nabla_\theta \mathscr D_{{\rm KL}}(q_{\theta^i_k}^1|P_0^1)\\
 &\theta_{0}^i=(m_{0}^i,\rho_{0}^i)\sim \mu_0,
\end{cases}\end{equation}
where $\eta>0$ and $(\mathsf Z^{p}_k, p\in \{1,2\}, k\ge 0)$ is a
i.i.d sequence of random variables distributed according to
$\gamma^{\otimes2}$.  We call this backpropagation scheme
\emph{minimal- VI SGD } which is much cheaper in terms of computational complexity, with the same limiting behavior as we now discuss.
\begin{figure}
    \centering
    \includegraphics[width=0.7\columnwidth]
    {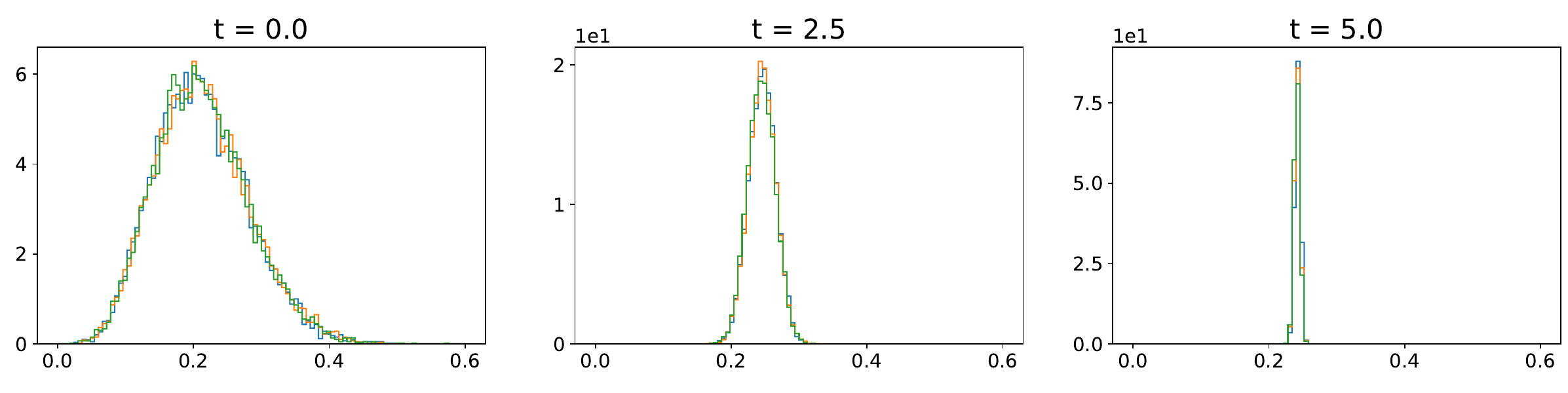}
    \includegraphics[width=0.7\columnwidth]{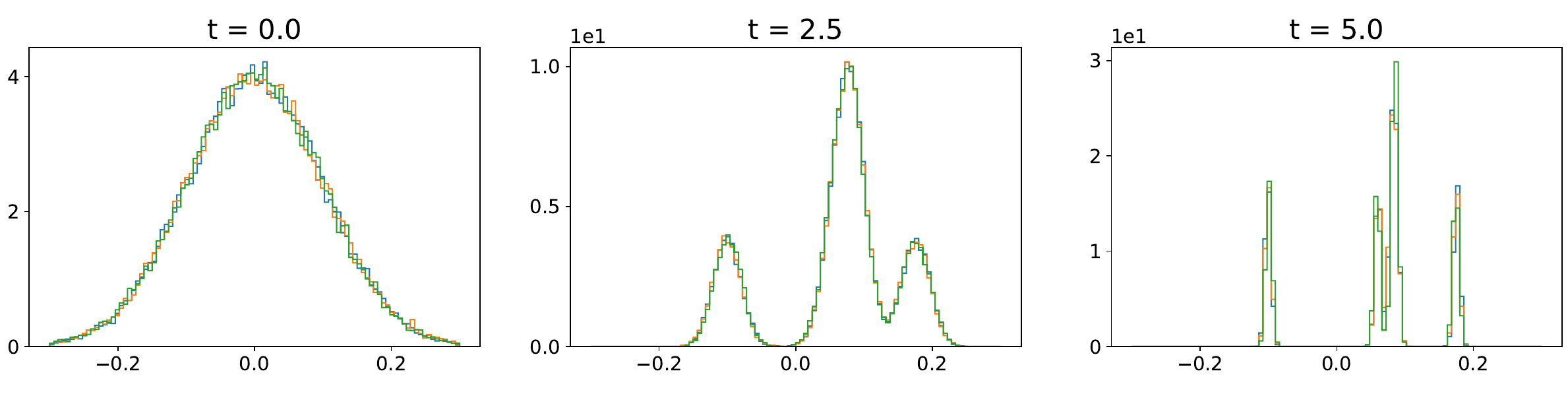}
    \caption{Histograms of $\{F(\theta_{\lfloor NT\rfloor}^i),  i=1,\dots,N\}$, at different times (initialization ($t=0$), half ($t=2.5$) and end of training ($T=5$)), when $N=10000$. First line: $F(\theta)=\|m\|_2$, where $\theta=(m,\rho)\in\mathbf R^d\times\mathbf R.$ Second line: $F(\theta)=m\in\mathbf R^d$. Idealized (blue), \emph{Bayes-by-Backprop} (orange) and \emph{Minimal-VI} (green).}
    \label{fig:hist}
\end{figure}

We introduce the  $\boldsymbol\sigma$-algebra for $N,k\ge 1$:
\begin{equation}\label{eq.Fk3}
\mathcal F_k^N=\boldsymbol\sigma \Big (\theta_{0}^i ,   \mathsf Z^{p}_q,(x_q,y_q),  1\leq i\leq N, p\in \{1,2\}, 0\le q\le k-1\big \} \Big ).
\end{equation}
In addition to  {\rm \textbf{A1}}$\to${\rm \textbf{A4}}  (where in {\rm \textbf{A2}},  $\mathcal F_k^N$ is now the one defined above  in~\eqref{eq.Fk3} when $k\ge 1$),  the following assumption
\begin{enumerate}
\item[\textbf{A6}.]   The sequences $(\mathsf Z^{p}_k, p\in \{1,2\},  k\ge 0)$ and $((x_k,y_k), k\ge 0)$ are independent. In addition, for $k\ge 0$, $\big((x_k,y_k),\mathsf Z^{p}_k, p\in \{1,2\}\big)\indep  \mathcal F_k^N$, where $\mathcal F_k^N$ is defined in~\eqref{eq.Fk3}.
 \end{enumerate}
Set   for $k\ge 0$ and $t\ge 0$, $\nu_k^N:=\frac 1N\sum_{i=1}^N\delta_{\theta_k^i}$ and $ \mu_t^N:=\nu_{\lfloor Nt\rfloor}^N$, where the  $\theta^i_k$'s are defined in \eqref{eq.algo-z1z2}.
The last main result of this work  states  that the  sequence  $(\mu^N)_{N\ge1}$ satisfies  the same law of large numbers  when $N\to +\infty$ as  the one satisfied by~\eqref{empirical_distrib2}, whose proof will be omitted as it is the same as the one made for  Theorem \ref{thm.z1zN}.

\begin{theorem}\label{thm.z1z2}
Assume  {\rm \textbf{A1}}$\to${\rm \textbf{A4}} and {\rm \textbf{A6}}.  Then, the sequence of $(\mu^N)_{N\ge1}$  satisfies all the statements of Theorem~\ref{thm.z1zN}.
\end{theorem}


\begin{figure}
    \centering
    \includegraphics[width=0.3\columnwidth]{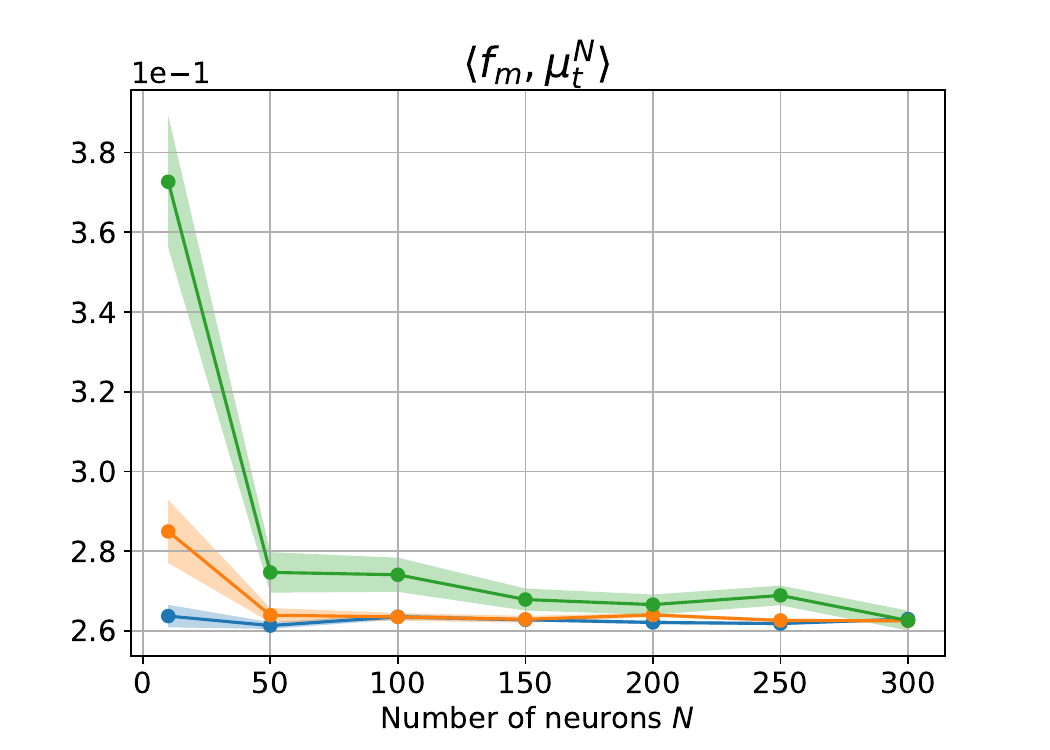}\hfill
    \includegraphics[width=0.3\columnwidth]{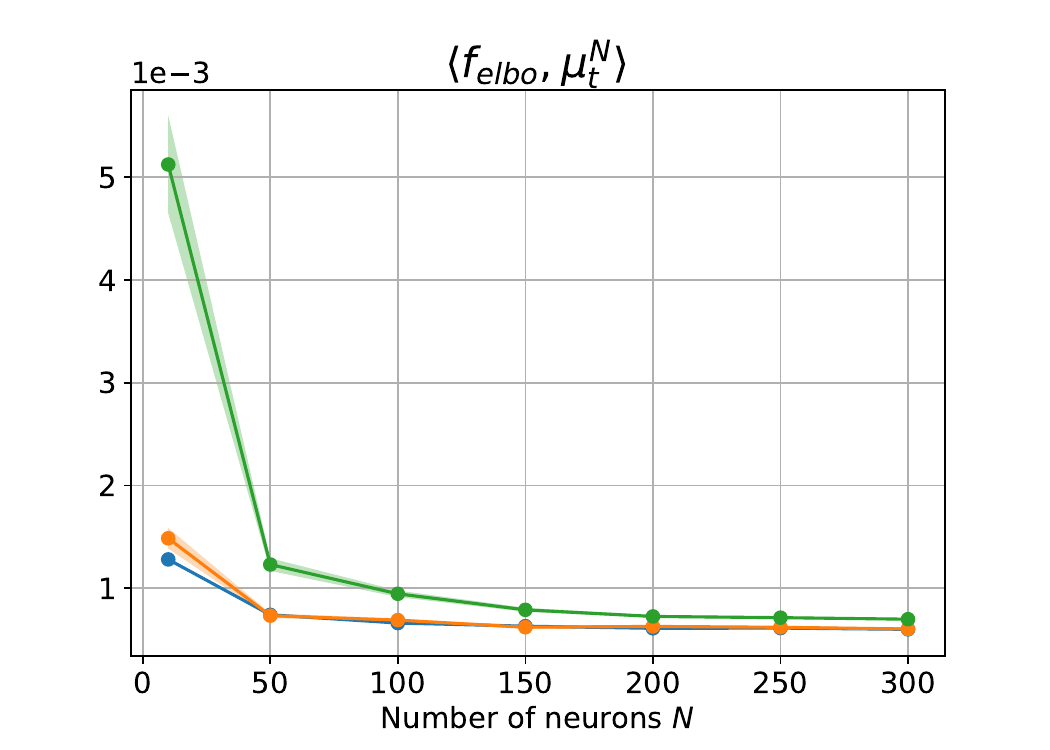}\hfill
    \includegraphics[width=0.3\columnwidth]{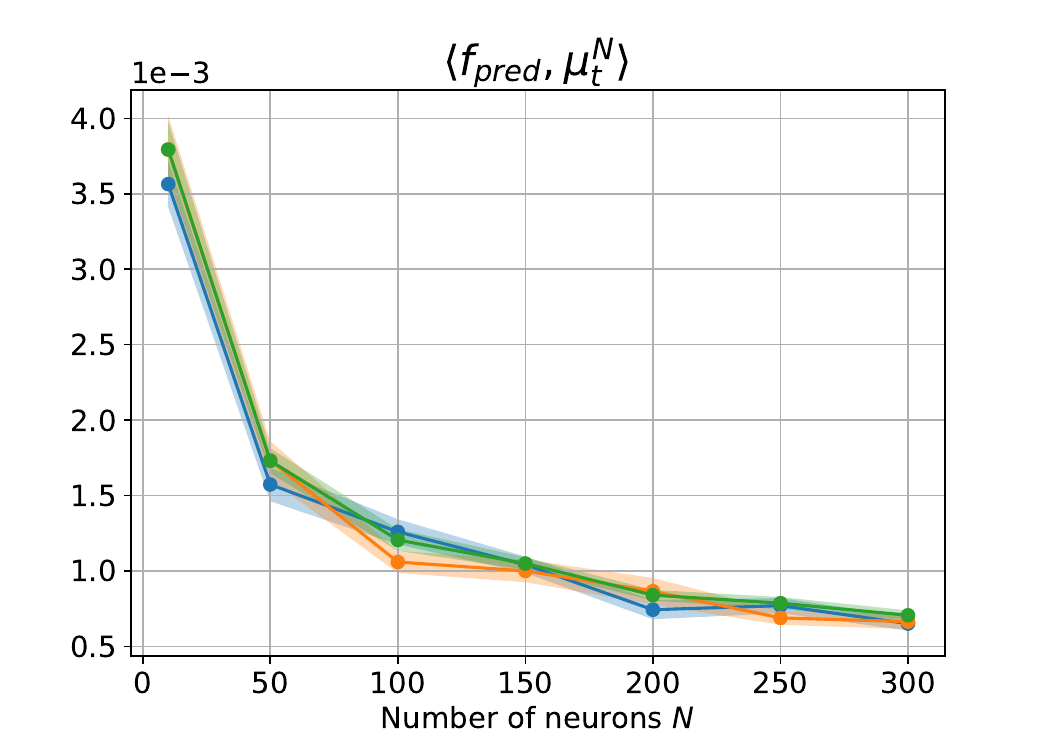}
    \caption{{  Convergence of $\langle f, \mu_T^{N}\rangle$ to $\langle f, \bar\mu_T\rangle$, for the idealized (blue), \emph{Bayes-by-Backprop} (orange) and \emph{Minimal-VI} (green) SGD algorithms over $50$ realizations}.}
    \label{fig:convergence}
\end{figure}

\begin{figure}
    \centering
    \includegraphics[width=0.6\columnwidth]{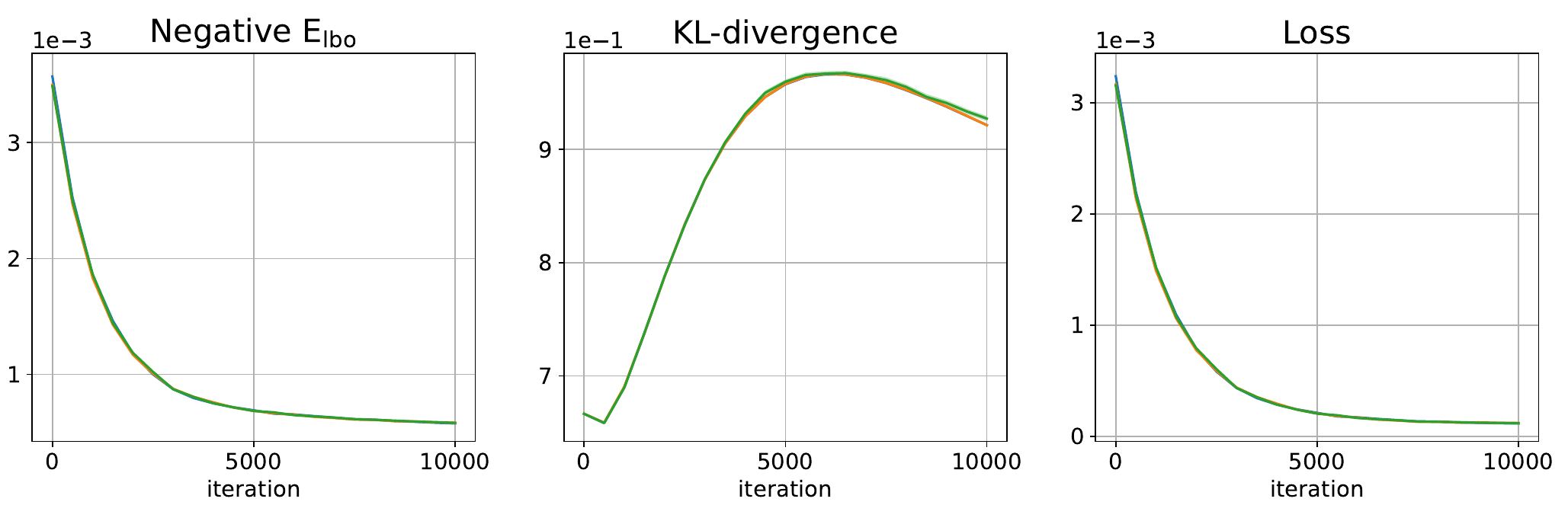}
    \caption{{  Decay of the negative ELBO (left) and its two components (KL (middle), loss (right)) during the training process done by the idealized (blue), \emph{Bayes-by-Backprop} (orange) and \emph{Minimal-VI} (green) SGD algorithms, for $N=10000$.}}
    \label{fig:convergence_time}
\end{figure}

\section{Numerical experiments}
\label{sec:numerics}
In this section we illustrate the theorems \ref{thm.ideal},
\ref{thm.z1zN}, and \ref{thm.z1z2} using the following toy model. We
set $d=5$. Given $\theta^*\in \mathbf R^d$ (drawn from a normal
distribution and scaled to the unit norm), we draw i.i.d observations
as follows: Given $x\sim \mathcal U([-1,1]^d)$, we draw
$y=\tanh(x^\top\theta^*)+\epsilon,$ where $\epsilon$ is zero mean with
variance $10^{-4}$. The initial distribution of parameters is centered
around the prior:
$\boldsymbol{\theta}_0\sim (\mathcal N(m_0,0.01I_d)\times \mathcal
N(g^{-1}(\sigma_0),0.01))^{\otimes N}$, {with $m_0=0$ and
  $\sigma_0=0.2$}. Since the idealized algorithm cannot be
implemented exactly, a mini-batch of size 100 is used as a proxy for
the following comparisons of the different algorithms. For the
algorithm \eqref{eq.algo-batch} SGD we set $B=1$.

\noindent{\bf Evolution and limit of the distribution}
Fig.~\ref{fig:hist} displays the histograms of
$\{F(\theta_{\lfloor Nt\rfloor}^i), i=1,\dots,N\}$
($F(\theta)\!=\!\|m\|_2, g(\rho)$ or $m$, where
$\theta=(m,\rho)\in\mathbf R^d\times\mathbf R$), for $N=10000$, at initialization, halfway through training, and at the end of training. The empirical distributions illustrated by these histograms are very similar over the course of training. It can be seen that for $N=10000$ the limit of the mean field is reached.

\noindent{\bf Convergence with respect to the numbers of neurons.}
We investigate here the speed of convergence of $\mu_t^N$ to
$\bar\mu_t$ (as $N\to+\infty$), when tested against test functions
$f$. More precisely, we fix a time $T$ (end of training) and Figure
\ref{fig:convergence} represents the empirical mean of
$\langle f, \mu_T^{N}\rangle$ over 50 realizations.  The test
functions used for this experiment are $f_m(\theta) = \Vert m\Vert_2$,
$f_{\mathrm {Elbo}(\theta)} = -
\hat{\mathrm{E}}_{\mathrm{lbo}}(\theta)^{N}$ where
$\hat{\mathrm{E}}_{\mathrm{lbo}}$ is the empirical
$\mathrm E_{\mathrm{lbo}}^N$ (see \eqref{d-elbo_w}) computed with 100
samples of $(x,y)$ and $(z^1,\dots,z^N)$. Finally,
$f_{pred}(\theta) =
\hat{\mathbb{E}}_x\Big[\hat{\mathbb{V}}_{\boldsymbol{w} \sim
  q_{\boldsymbol{\theta}}^N}{[f_{\boldsymbol{w}}^N(x)]}^{1/2}\Big]$
where $\hat{\mathbb{E}}$ and $\hat{\mathbb V}$ denote respectively the
empirical mean and the empirical variance over 100 samples.  All
algorithms are converging to the same limit and are performing similarly
even with a limited number of neurons ($N=300$ in this example).


\noindent{\bf Convergence with respect to time.}
This section illustrates the training process of a BNN with a given
number of neurons $N = 10000$. In Figure \ref{fig:convergence_time},
we plot the negative ELBO on a test set and its two components, the
loss and the KL-divergence terms. Figure \ref{fig:convergence_time}
shows that the BNN is able to learn on this specific task and all
algorithms exhibit a similar performance. It illustrates the
trajectorial convergence of $\{\mu_t^N, t\in[0,T]\}_{N\ge1}$ to
$\{\bar\mu_t, t\in[0,T]\}$ as $N\to+\infty$.

\noindent{\bf Behavior around the limit $\bar\mu$. }
On Figure \ref{fig:boxplots}, we plot the boxplots
of $\langle f,\mu_t^N\rangle$ for 50 realizations and $N=10000$, at
different times of the training.
\emph{Minimal-VI} scheme (which is computationally cheaper as
explained in \ref{sec:minimal}) exhibit a larger variance than the
other algorithms.

\begin{figure}
    \centering
    \includegraphics[width=0.85\columnwidth]{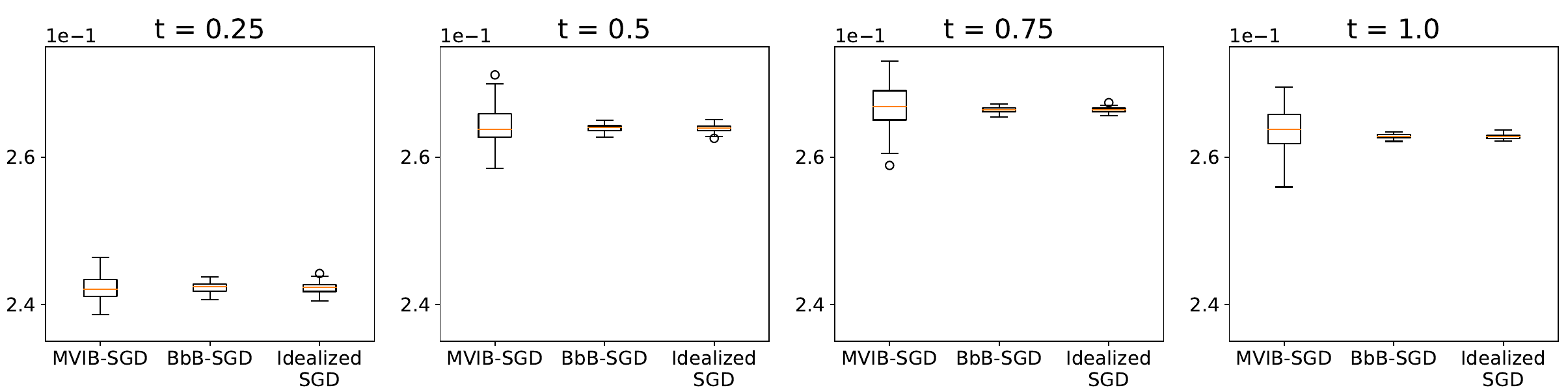}
    \includegraphics[width=0.85\columnwidth]{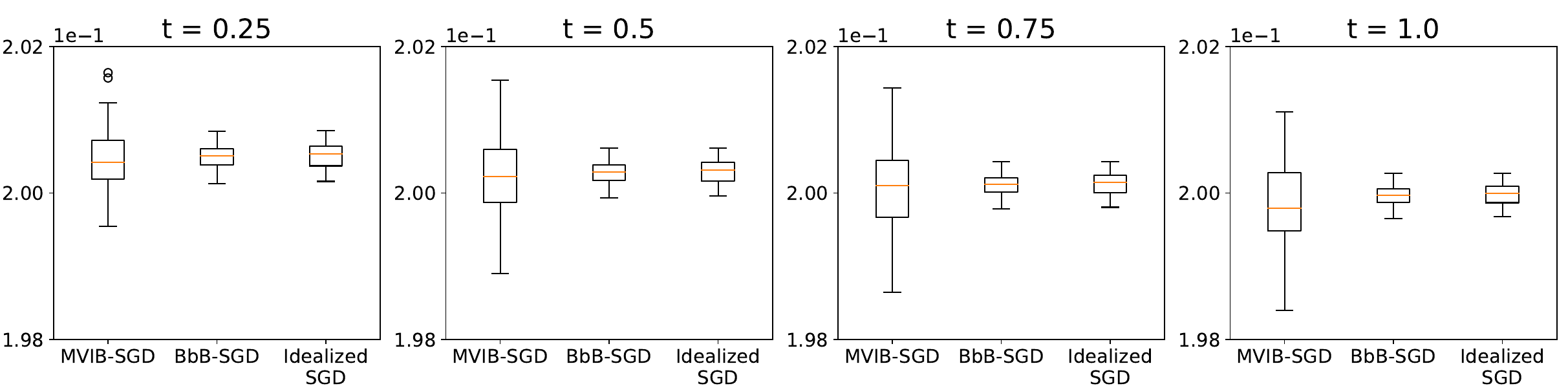}
    \caption{{  Boxplots for $50$ runs of $\langle f,\mu_t^N\rangle$ for the three SGD schemes for $f(\theta)=\|m\|_2$ on the first line and $f(\theta)=g(\rho)$ on the second line. MVIB-SGD: \textit{Minimal-VI} SGD. BbB-SGD: \textit{Bayes-by-Backprop} SGD}.}
    \label{fig:boxplots}
\end{figure}

\section{Conclusion}
By establishing the limit behavior of the idealized
SGD for the variational inference of BNN with the weighting suggested
by \cite{huix}, we have rigorously shown that the most-commonly used
in practice \emph{Bayes-by-Backprop} scheme indeed exhibits the same
limit behavior. {Furthermore, the analysis of the limit equation
  led us to validate the correct scaling of the KL divergence term in
  with respect to the loss}. {Notably, the mean-field limit dynamics
  has also helped us to devise a far less costly new SGD algorithm,
  the \emph{Minimal-VI}. This scheme shares the same limit
  behavior, but only stems from the non-vanishing asymptotic
  contributions, hence the reduction of the computational cost}. Aside
from confirming the analytical results, the first simulations
presented here show that the three algorithms, while having the same
limit, may differ in terms of variance. Thus, deriving a CLT result
and discussing the right trade-off between computational complexity
and variance will be done in future work. {Also, on a more
  general level regarding uncertainty quantification, an interesting
  question is to analyse the impact of the correct scaling of the KL
  divergence term on the error calibration and how to apply the same
  analysis in the context of deep ensembles.\\}

\acks{A.D. is grateful for the
support received from the Agence Nationale de la Recherche (ANR) of
the French government through the program "Investissements d'Avenir"
(16-IDEX-0001 CAP 20-25)   A.G. is supported by the French ANR under
the grant ANR-17-CE40-0030 (project \emph{EFI}) and the Institut
Universtaire de France. M.M. acknowledges the support of the the
French ANR under the grant ANR-20-CE46-0007 (\emph{SuSa} project).
B.N. is supported by the grant IA20Nectoux from the Projet I-SITE
Clermont CAP 20-25. E.M. and T.H. acknowledge the support of ANR-CHIA-002, "Statistics, computation and Artificial Intelligence"; Part of the work has been developed under the auspice of the Lagrange Center for Mathematics and Calculus}

\bibliography{elbo_bib}

\clearpage

\appendix

\section{Proof of Theorem \ref{thm.ideal}}
\label{sec.proof1}
For simplicity, we prove the theorem \ref{thm.ideal} when $T=1$, and we denote~$\Theta_1$ simply by~$\Theta$. In this section we assume  \textbf{A1}--\textbf{A4}.
\subsection{Pre-limit equation~\eqref{eq.pre_limit} and  error terms in~\eqref{eq.pre_limit}}
\label{sec.pre-eq-1}

\subsubsection{Derivation of the pre-limit equation}\label{sec:pre_limit_eq}
\label{sec.PR1}


The aim of this section is to establish the so-called pre-limit equation~\eqref{eq.pre_limit}, which will be our starting point to derive Equation~\eqref{eq_limit}.
Let $N\ge 1$, $k\in \{0,\ldots,N\}$, and $f\in\mathcal C^\infty(\Theta)$.     Recall that by Lemma~\ref{lem:unif_bound_param} and since $0\le k \le N$, a.s. $\theta^i_k\in \Theta$, and thus a.s. $f(\theta^i_k)$ is well-defined. The Taylor-Lagrange formula yields
\begin{align*}
\langle f,\nu_{k+1}^N\rangle-\langle f,\nu_k^N\rangle&=\frac{1}{N}\sum_{i=1}^Nf(\theta_{k+1}^i)-f(\theta_k^i)\nonumber\\
&=\frac 1N\sum_{i=1}^N\nabla_\theta f(\theta_k^i)\cdot(\theta_{k+1}^i-\theta_k^i) +\frac{1}{2N}\sum_{i=1}^N(\theta_{k+1}^i-\theta_k^i)^T\nabla^2f(\widehat{\theta_k^i})(\theta_{k+1}^i-\theta_k^i),
\end{align*}
where, for all $i\in \{1,\dots, N\}$, $\widehat{\theta_k^i}\in (\theta_k^i,\theta_{k+1}^i)\subset \Theta$.
Using \eqref{eq.algo-ideal}, we then obtain
\begin{align}\label{eq_1}
\langle f,\nu_{k+1}^N\rangle-\langle f,\nu_k^N\rangle&=
-\frac{\eta}{N^3}\sum_{i=1}^N\sum_{j=1,j\neq i}^N\big (\big \langle\phi(\theta_k^j,\cdot,x_k),\gamma\big \rangle-y_k\big )\big \langle\nabla_\theta f(\theta_k^i)\cdot\nabla_\theta  \phi(\theta_k^i,\cdot,x_k),\gamma\big \rangle\nonumber\\
&\quad-\frac{\eta}{N^2}\big \langle\big(\phi(\cdot,\cdot,x_k)-y_k\big)\nabla_\theta f\cdot\nabla_\theta \phi(\cdot,\cdot,x_k),\nu_k^N\otimes\gamma\big \rangle\nonumber\\
%
&\quad-\frac{\eta}{N}\big \langle\nabla_\theta f\cdot \nabla_\theta \mathscr D_{{\rm KL}}(q_{\,_\cdot }^1|P_0^1),\nu_k^N\big \rangle + \mathbf R_k^N[f],
\end{align}
where
\begin{equation}\label{def_R_k^N}
\mathbf R_k^N[f]:=\frac{1}{2N}\sum_{i=1}^N(\theta_{k+1}^i-\theta_k^i)^T\nabla^2f(\widehat{\theta_k^i})(\theta_{k+1}^i-\theta_k^i).
\end{equation}
Let us define
\begin{align}\label{d-D}
\mathbf D_{k}^N[f]&:= \mathbf E\Big[-\frac{\eta}{N^3}\sum_{i=1}^N\sum_{j=1,j\neq i}^N\big (\big \langle\phi(\theta_k^j,\cdot,x_k),\gamma\big \rangle-y_k\big )\big \langle\nabla_\theta f(\theta_k^i)\cdot\nabla_\theta \phi(\theta_k^i,\cdot,x_k),\gamma\big \rangle\Big|\mathcal F_k^N\Big]\nonumber\\
&\quad-\mathbf E\Big[\frac{\eta}{N^2}\big \langle(\phi(\cdot,\cdot,x_k)-y_k)\nabla_\theta f\cdot\nabla_\theta \phi(\cdot,\cdot,x_k),\nu_k^N\otimes\gamma\big \rangle\Big|\mathcal F_k^N\Big].
\end{align}
Note that using     \eqref{borne_phi-y} and \eqref{borne nablaphi} together with the fact that $|\nabla_\theta f(\theta_k^i)|\le \sup_{\theta \in \Theta} |\nabla_\theta f(\theta)|$,
the integrant in \eqref{d-D} is integrable and thus $\mathbf D_{k}^N[f]$ is well defined.
Using  the fact that $(x_k,y_k)\indep \mathcal F_{k}^N$ by {\rm \textbf{A2}}  and that $\{\theta_k^i, i=1,\ldots,N\}$ is $\mathcal F_k^N$-measurable by \eqref{eq.algo-ideal}, we have:
\begin{align}\label{D_k-calcule}
\mathbf D_{k}^N[f]&=-\frac{\eta}{N^3}\sum_{i=1}^N\sum_{j=1,j\neq i}^N\int_{\mathsf X\times\mathsf Y}\big (\big \langle\phi(\theta_k^j,\cdot,x),\gamma\big \rangle-y\big )\big \langle\nabla_\theta f(\theta_k^i)\cdot\nabla_\theta\phi(\theta_k^i,\cdot,x),\gamma\big \rangle\pi(\di x,\di y)\nonumber\\
&\quad-\frac{\eta}{N^2}\int_{\mathsf X\times\mathsf Y}\big \langle(\phi(\cdot,\cdot,x)-y)\nabla_\theta f\cdot\nabla_\theta\phi(\cdot,\cdot,x),\nu_k^N\otimes\gamma\big \rangle\pi(\di x,\di y).
\end{align}
Introduce also
\begin{align*}
\mathbf  M_{k}^{N}[f]&:=-\frac{\eta}{N^3}\sum_{i=1}^N\sum_{j=1,j\neq i}^N(\langle\phi(\theta_k^j,\cdot,x_k),\gamma\rangle-y_k)\langle\nabla_ \theta f(\theta_k^i)\cdot\nabla_\theta\phi(\theta_k^i,\cdot,x_k),\gamma\rangle\nonumber\\
&\quad-\frac{\eta}{N^2}\langle(\phi(\cdot,\cdot,x_k)-y_k)\nabla_\theta f\cdot\nabla_\theta\phi(\cdot,\cdot,x_k),\nu_k^N\otimes\gamma\rangle-\mathbf D_{k}^N[f].
\end{align*}
Note that $\mathbf E\big [\mathbf M_{k}^{N}[f]|\mathcal F_k^N\big]=0$. Equation \eqref{eq_1} then writes
\begin{align}\label{eq_2}
\langle f,\nu_{k+1}^N\rangle-\langle f,\nu_k^N\rangle&=\mathbf D_{k}^N[f]+ \mathbf  M_{k}^N[f] -\frac{\eta}{N}\big \langle\nabla_\theta f\cdot \nabla_\theta \mathscr D_{{\rm KL}}(q_{\,_\cdot }^1|P_0^1),\nu_k^N\big \rangle +\mathbf R_k^N[f].
\end{align}
Notice also that
\begin{align}\label{calc_fD^1}
\mathbf  D_{k}^N[f]&=-\frac{\eta}{N^3}\sum_{i=1}^N\sum_{j=1}^N\int_{\mathsf X\times\mathsf Y}(\langle\phi(\theta_k^j,\cdot,x),\gamma\rangle-y)\langle\nabla_\theta f(\theta_k^i)\cdot\nabla_\theta \phi(\theta_k^i,\cdot,x),\gamma\rangle\pi(\di x,\di y) \nonumber\\
&\quad+\frac{\eta}{N^3}\sum_{i=1}^N\int_{\mathsf X\times\mathsf Y}(\langle\phi(\theta_k^i,\cdot,x),\gamma\rangle-y)\langle\nabla_\theta f(\theta_k^i)\cdot\nabla_\theta \phi(\theta_k^i,\cdot,x),\gamma\rangle\pi(\di x,\di y)\nonumber\\
&\quad-\frac{\eta}{N^2}\int_{\mathsf X\times\mathsf Y}\langle(\phi(\cdot,\cdot,x)-y)\nabla_\theta f\cdot\nabla_\theta \phi(\cdot,\cdot,x),\nu_k^N\otimes\gamma\rangle\pi(\di x,\di y)\nonumber\\
&=-\frac{\eta}{N}\int_{\mathsf X\times\mathsf Y}\langle\phi(\cdot,\cdot,x)-y,\nu_k^N\otimes\gamma\rangle\langle\nabla_\theta  f\cdot\nabla_\theta \phi(\cdot,\cdot,x),\nu_k^N\otimes\gamma\rangle\pi(\di x,\di y)\nonumber \\
&\quad+\frac{\eta}{N^2}\int_{\mathsf X\times\mathsf Y}\Big\langle(\langle\phi(\cdot,\cdot,x),\gamma\rangle-y)\langle\nabla_\theta f\cdot\nabla_\theta \phi(\cdot,\cdot,x),\gamma\rangle,\nu_k^N\Big\rangle\pi(\di x,\di y)  \nonumber\\
&\quad-\frac{\eta}{N^2}\int_{\mathsf X\times\mathsf Y}\langle(\phi(\cdot,\cdot,x)-y)\nabla_\theta f\cdot\nabla_\theta \phi(\cdot,\cdot,x),\nu_k^N\otimes\gamma\rangle\pi(\di x,\di y).
\end{align}
 Now, we define for   $t\in [0,1]$:
\begin{align}\label{def D_t}
\mathbf D_t^{N}[f]:=\sum_{k=0}^{\lfloor Nt\rfloor-1} \mathbf  D_k^{N}[f], \  \  \mathbf R_t^N[f]:=\sum_{k=0}^{\lfloor Nt\rfloor-1}\mathbf R_k^N[f],  \ \ \text{and}\ \ \mathbf  M_t^{N}[f]:=\sum_{k=0}^{\lfloor Nt\rfloor-1}\mathbf M_{k}^{N}[f] .
\end{align}
We can rewrite $\mathbf D_t^{N}[f]$ has follows:
\begin{align*}
\mathbf  D_t^{N}[f]=\sum_{k=0}^{\lfloor Nt\rfloor-1}\int_{\frac{k}{N}}^{\frac{k+1}{N}}N \mathbf  D_{\lfloor Ns\rfloor}^{N}[f]\di s=N\int_0^t \mathbf D_{\lfloor Ns\rfloor}^{N}[f]\di s-N\int_{\frac{\lfloor Nt\rfloor}{N}}^t  \mathbf  D_{\lfloor Ns\rfloor}^{N}[f]\di s.
\end{align*}
Since $\nu_{\lfloor Ns\rfloor}^N=\mu_s^N$ (by  definition, see \eqref{empirical_distrib1}), we have, using also  \eqref{calc_fD^1} with $k=\lfloor Ns\rfloor$,
\begin{align}\label{calc_D_t}
 \mathbf D_t^{N}[f]&=-\eta\int_{0}^t\int_{\mathsf X\times\mathsf Y}\langle\phi(\cdot,\cdot,x)-y,\mu_s^N\otimes\gamma\rangle\langle\nabla_\theta f\cdot\nabla_\theta \phi(\cdot,\cdot,x),\mu_s^N\otimes\gamma\rangle  \pi(\di x,\di y)\di s\nonumber\\
&\quad +\frac{\eta}{N}\int_{0}^t\int_{\mathsf X\times\mathsf Y} \Big\langle\langle\phi(\cdot,\cdot,x)-y,\gamma\rangle\langle\nabla_\theta f\cdot\nabla_\theta \phi(\cdot,\cdot,x),\gamma\rangle,\mu_s^N\Big\rangle \pi(\di x,\di y)\di s\nonumber\\
&\quad -\frac{\eta}{N}\int_{0}^t\int_{\mathsf X\times\mathsf Y} \Big\langle(\phi(\cdot,\cdot,x)-y)\nabla_\theta f\cdot\nabla_\theta \phi(\cdot,\cdot,x),\mu_s^N\otimes\gamma\Big\rangle \pi(\di x,\di y)\di s-\mathbf V_t^{N}[f],
\end{align}
where
\begin{align*}
\mathbf V_t^{N}[f]&:=-\eta\int^{t}_{\frac{\lfloor Nt\rfloor}{N}}\int_{\mathsf X\times\mathsf Y}\langle\phi(\cdot,\cdot,x)-y,\mu_s^N\otimes\gamma\rangle\langle\nabla_\theta f\cdot\nabla_\theta \phi(\cdot,\cdot,x),\mu_s^N\otimes\gamma\rangle  \pi(\di x,\di y)\di s\\
&\quad +\frac{\eta}{N}\int^{t}_{\frac{\lfloor Nt\rfloor}{N}}\int_{\mathsf X\times\mathsf Y} \Big\langle\langle\phi(\cdot,\cdot,x)-y,\gamma\rangle\langle\nabla_\theta f\cdot\nabla_\theta \phi(\cdot,\cdot,x),\gamma\rangle,\mu_s^N\Big\rangle \pi(\di x,\di y)\di s\\
&\quad -\frac{\eta}{N}\int^{t}_{\frac{\lfloor Nt\rfloor}{N}}\int_{\mathsf X\times\mathsf Y} \Big\langle(\phi(\cdot,\cdot,x)-y)\nabla_\theta f\cdot\nabla_\theta \phi(\cdot,\cdot,x),\mu_s^N\otimes\gamma\Big\rangle \pi(\di x,\di y)\di s.
\end{align*}
On the other hand, we also have for $t\in [0,1]$,
\begin{align}\label{calc_KL}
&\sum_{k=0}^{\lfloor Nt\rfloor-1}-\frac{\eta}{N}\big \langle\nabla_\theta f\cdot \nabla_\theta \mathscr D_{{\rm KL}}(q_{\,_\cdot }^1|P_0^1),\nu_k^N\big \rangle   =-\eta \int_0^{\frac{\lfloor Nt\rfloor}{N}} \big \langle\nabla_\theta f\cdot \nabla_\theta \mathscr D_{{\rm KL}}(q_{\,_\cdot }^1|P_0^1),\mu_s^N\big \rangle\di s.
\end{align}
We finally set:
\begin{align}\label{def V_t^N}
 \mathbf  W_t^{N}[f]:=-  \mathbf  V_t^{N}[f] + \eta \int^t_{\frac{\lfloor Nt\rfloor}{N}}\big \langle\nabla_\theta f\cdot \nabla_\theta \mathscr D_{{\rm KL}}(q_{\,_\cdot }^1|P_0^1),\mu_s^N\big \rangle\di s.
\end{align}
Since $\langle f,\mu_t^N\rangle-\langle f,\mu_0^N\rangle=\sum_{k=0}^{\lfloor Nt\rfloor-1}\langle f,\nu_{k+1}^N\rangle-\langle f,\nu_k^N\rangle$,  we deduce from \eqref{eq_2}, \eqref{def D_t}, \eqref{calc_D_t}, \eqref{calc_KL}   and \eqref{def V_t^N},   the so-called pre-limit equation satisfied by
$\mu^N$:  for $N\ge1$, $t\in [0,1]$, and $f\in \mathcal C^\infty(\Theta)$,
\begin{align}\label{eq.pre_limit}
\langle f,\mu_t^N\rangle-\langle f,\mu_0^N\rangle&=-\eta\int_{0}^t\int_{\mathsf X\times\mathsf Y}\langle\phi(\cdot,\cdot,x)-y,\mu_s^N\otimes\gamma\rangle\langle\nabla_\theta f\cdot\nabla_\theta \phi(\cdot,\cdot,x),\mu_s^N\otimes\gamma\rangle  \pi(\di x,\di y)\di s\nonumber\\
&\quad -\eta \int_0^t \big \langle\nabla_\theta f\cdot \nabla_\theta \mathscr D_{{\rm KL}}(q_{\,_\cdot }^1|P_0^1),\mu_s^N\big \rangle\di s  \nonumber\\
&\quad +\frac{\eta}{N}\int_{0}^t\int_{\mathsf X\times\mathsf Y} \Big\langle\langle\phi(\cdot,\cdot,x)-y,\gamma\rangle\langle\nabla_\theta  f\cdot\nabla_\theta \phi(\cdot,\cdot,x),\gamma\rangle,\mu_s^N\Big\rangle \pi(\di x,\di y)\di s\nonumber\\
&\quad -\frac{\eta}{N}\int_{0}^t\int_{\mathsf X\times\mathsf Y} \Big\langle(\phi(\cdot,\cdot,x)-y)\nabla_\theta f\cdot\nabla_\theta \phi(\cdot,\cdot,x),\mu_s^N\otimes\gamma\Big\rangle \pi(\di x,\di y)\di s\nonumber\\
&\quad +  \mathbf M_t^{N}[f] +\mathbf W_t^{N}[f]+ \mathbf R_t^N[f].
\end{align}


\subsubsection{The last five  terms in \eqref{eq.pre_limit} are error terms}


The purpose of this section is to show that the last five  terms appearing in the r.h.s. of \eqref{eq.pre_limit} are error terms  when  $N\to+\infty$.
For $J\in \mathbf N^*$ and $f\in\mathcal C^J(\Theta)$, set $\Vert f\Vert_{\mathcal C^J(\Theta)}:=\sum_{|k|\leq J}\Vert \partial_kf \Vert_{\infty, \Theta}$,
where $\Vert g\Vert_{\infty, \Theta}=\sup_{\theta\in \Theta}|g(\theta)|$ for $g:\Theta\to \mathbf R^m$.

\begin{lemma}[Error terms]\label{le.error}
Assume  {\rm \textbf{A1}}$\to${\rm \textbf{A4}}. Then, there exists $C>0$  such that a.s.  for all $f\in\mathcal C^\infty(\Theta)$ and $N\ge1$,
\begin{enumerate}
\item\label{remTi1}  $\frac{\eta}{N}\int_{0}^1\int_{\mathsf X\times\mathsf Y} \Big|\Big\langle\langle\phi(\cdot,\cdot,x)-y,\gamma\rangle\langle\nabla_\theta f\cdot\nabla_\theta \phi(\cdot,\cdot,x),\gamma\rangle,\mu_s^N\Big\rangle\Big| \pi(\di x,\di y)\di s \leq C\Vert f\Vert_{\mathcal C^1(\Theta)}/N$.

\item\label{remTi2}  $\frac{\eta}{N}\int_{0}^1\int_{\mathsf X\times\mathsf Y} \Big|\Big\langle(\phi(\cdot,\cdot,x)-y)\nabla_\theta f\cdot\nabla_\theta \phi(\cdot,\cdot,x),\mu_s^N\otimes\gamma\Big\rangle\Big| \pi(\di x,\di y)\di s\leq  C\Vert f\Vert_{\mathcal C^1(\Theta)}/N$.
\item\label{remTi6} $\sup_{t\in[0,1]}|\mathbf W_t^{N}[f]|+ \sup_{t\in[0,1]}|\mathbf R_t^N[f]| \leq   C\Vert f\Vert_{\mathcal C^2(\Theta)}/N$.
\end{enumerate}
Finally, $\sup_{t\in[0,1]}\mathbf E\big[|\mathbf M_t^{N}[f]|\big]\leq  {C}\Vert f\Vert_{\mathcal C^1(\Theta)}/{\sqrt N}$.
\end{lemma}
\begin{proof}   All along the proof, $C>0$ denotes a positive constant independent of $N\ge 1,k\in \{0,\ldots,N-1\},(s,t)\in [0,1]^2,(x,y)\in \mathsf X\times \mathsf Y,\theta\in \Theta,z\in \mathbf R^d$, and  $f\in\mathcal C^\infty(\Theta)$ which can change from one occurrence to another.
Using \eqref{borne nablaphi}, the Cauchy-Schwarz inequality, and the  fact that $\nabla_\theta f$ is bounded over  $\Theta$ imply:
\begin{align}\label{borne_nabla}
|\langle\nabla_\theta f(\theta)\cdot\nabla_\theta\phi(\theta,\cdot,x),\gamma\rangle|\le \langle|\nabla_\theta f(\theta)\cdot\nabla_\theta\phi(\theta,\cdot,x)|,\gamma\rangle \leq C\Vert f\Vert_{\mathcal C^1(\Theta)}.
\end{align}
Combining \eqref{borne_phi-y} and \eqref{borne_nabla}, we obtain:
\begin{align*}
\int_{0}^1\int_{\mathsf X\times\mathsf Y} \Big|\Big\langle\langle\phi(\cdot,\cdot,x)-y,\gamma\rangle\langle\nabla_\theta f\cdot\nabla_\theta \phi(\cdot,\cdot,x),\gamma\rangle,\mu_s^N\Big\rangle\Big| \pi(\di x,\di y)\di s\leq C\Vert f\Vert_{\mathcal C^1(\Theta)}
\end{align*}
and
\begin{align*}
\int_{0}^1\int_{\mathsf X\times\mathsf Y} \Big|\Big\langle(\phi(\cdot,\cdot,x)-y)\nabla_mf\cdot\nabla_m\phi(\cdot,\cdot,x),\mu_s^N\otimes\gamma\Big\rangle\Big| \pi(\di x,\di y)\di s\leq C\Vert f\Vert_{\mathcal C^1(\Theta)},
\end{align*}
which proves Items \ref{remTi1} and \ref{remTi2}.

 Let us now prove Item \ref{remTi6}. By \eqref{borne_phi-y} and \eqref{borne_nabla}, $\sup_{t\in[0,1]}|\mathbf V_t^{N}[f]|\leq C\Vert f\Vert_{\mathcal C^1(\Theta)}/N$.
On the other hand,
   because $f\in \mathcal C^\infty(\Theta)$ and $\theta\mapsto \nabla_\theta \mathscr D_{{\rm KL}}(q_{\theta }^1|P_0^1)$ is continuous (see \eqref{eq.kl_1}) over  $\Theta$ which is compact, it holds,  $ \Vert \nabla_\theta f\cdot \nabla_\theta \mathscr D_{{\rm KL}}(q_{\theta }^1|P_0^1)\Vert_{\infty,\Theta}<+\infty$.
%
   Hence, it holds:
\begin{align*}
\sup_{t\in[0,1]}\Big|\int^t_{\frac{\lfloor Nt\rfloor}{N}}\big \langle\nabla_\theta f\cdot \nabla_\theta \mathscr D_{{\rm KL}}(q_{\,_\cdot }^1|P_0^1),\mu_s^N\big \rangle\di s\Big|\leq C\Vert f\Vert_{\mathcal C^1(\Theta)}/N.
\end{align*}
Using \eqref{def V_t^N}, it then holds   $\sup_{t\in[0,1]}|\mathbf W_t^{N}[f]| \leq   C\Vert f\Vert_{\mathcal C^1(\Theta)}/N$.
  Since $f\in\mathcal C^\infty(\Theta)$, we have, by \eqref{def_R_k^N}, for $N\ge 1$ and $0\leq k\leq N-1$, $|\mathbf R_k^N[f]|\leq \Vert f\Vert_{\mathcal C^2(\Theta)}\frac CN\sum_{i=1}^N|\theta_{k+1}^i-\theta_k^i|^2$.
By \eqref{m^l+1-m^l} and Lemma \ref{lem:unif_bound_param},   $|\theta_{k+1}^i-\theta_k^i|^2\leq C/N^2$ and consequently,   one has: \begin{equation}\label{boundR_k}
|\mathbf R_k^N[f]|\leq {C}\Vert f\Vert_{\mathcal C^2(\Theta)}/{N^2}.
\end{equation}
Hence, for all $t\in[0,1]$, $|\mathbf R_t^N[f]|\leq  {C}\Vert f\Vert_{\mathcal C^2(\Theta)}/{N}$.
This proves Item \ref{remTi6}.

Let us now prove the last item in Lemma \ref{le.error}.  Let $t\in[0,1]$. We have, by \eqref{def D_t},
\begin{align*}
|\mathbf M_t^{N}[f]|^2=\sum_{k=0}^{\lfloor Nt\rfloor-1}\big |\mathbf M_{k}^{N}[f] \big |^2+2\sum_{k<j} \mathbf M_{k}^{N}[f] \,  \mathbf M_{j}^{N}[f]  .
\end{align*}
For all $0\leq k<j<\lfloor Nt\rfloor$, $\mathbf M_{k}^{N}[f]$ is $\mathcal F_{j}^N$-measurable (see \eqref{eq.Fk1}), and since $\mathbf E\big [\mathbf M_{j}^{N}[f]|\mathcal F_j^N\big]=0$, one deduces that  $\mathbf E\big [\, \mathbf M_{k}^{N}[f] \,  \mathbf M_{j}^{N}[f]\, \big ]=\mathbf E\big [\mathbf M_{k}^{N}[f] \, \mathbf E\big [\mathbf M_{j}^{N}[f]|\mathcal F_j^N\big]\, \big ]=0$.
Hence,  $\mathbf E[|\mathbf M_t^{N}[f]|^2]=\sum_{k=0}^{\lfloor Nt\rfloor-1} \mathbf E[|\mathbf M_k^{N}[f]|^2]$.
By \eqref{borne_phi-y} and \eqref{borne_nabla}, one has a.s. for all $0\leq k\leq N-1$,
\begin{equation}\label{bound M_k}
 |\mathbf M_{k}^{N}[f]|\leq C\Vert f\Vert_{\mathcal C^1(\Theta)}/N.
\end{equation}
   Hence, $
\mathbf E[|\mathbf M_t^{N}[f]|^2]\leq  C\Vert f\Vert_{\mathcal C^1(\Theta)}/N$, which proves the last inequality in Lemma \ref{le.error}.
\end{proof}



\subsection{Convergence to the limit equation as $N\to+\infty$}


In this section we prove the relative compactness of $(\mu^N)_{N\ge 1}$   in  $\mathcal D([0,1],\mathcal P(\Theta))$. We then show that any of its limit points satisfies the limit equation \eqref{eq_limit}.

 \subsubsection{Wasserstein spaces and duality formula}
In this section we recall some basic  results which will be used throughout this work on the space   $\mathcal P(\mathcal S)$ when $(\mathcal S, \mathsf d)$ is a  Polish space. First  when endowed with the weak convergence topology, $\mathcal P(\mathcal S)$   is a   Polish space~\cite[Theorem~6.8]{billingsley2013convergence}. In addition, $\mathcal P_q(\mathcal S)=  \{ \nu  \in  \mathcal P(\mathcal S), \int_{\mathcal S} \mathsf d(w_0,w)^q \nu (\di w)<+\infty \}$, where $w_0\in \mathcal S$ is arbitrary (note that this space   was defined previously  in \eqref{eq.Wq} when $\mathcal S=\mathbf R^{d+1}$) when endowed with the $\mathsf W_q$ metric is also a Polish space~\cite[Theorem 6.18]{villani2009optimal}. 
Recall also the duality formula for the $\mathsf W_1$-distance on $\mathcal P_1(\mathcal  S)$ (see e.g~\cite[Remark 6.5]{villani2009optimal}):
\begin{equation}\label{Kantorovitch Rubinstein}
\mathsf W_1(\mu,\nu)=\sup\Big \lbrace\big|\int_{\mathcal S}f(w)\di\mu(w)-\int_{\mathcal  S}f(w)\nu(\di w)\big|, \  \|f\|_{\text{Lip}}\leq 1\Big\rbrace.
\end{equation}
Finally, when $\mathcal K\subset \mathbf R^{d+1}$ is compact, the  convergence in $\mathsf W_q$-distance is equivalent to the usual weak convergence on $\mathcal P(\mathcal K)$ (see e.g.~\cite[Corollary 6.13]{villani2009optimal}).

%
%
%


\subsubsection{Relative compactness}
\label{sec.RC-1}



The main result of this section is to prove that  $(\mu^N)_{N\ge 1}$ is relatively compact in $\mathcal D([0,1],\mathcal P(\Theta))$, which is the purpose of    Proposition~\ref{p-rc} below. To this end, we need to prove that for all $f\in\mathcal C^\infty(\Theta)$,  every sequence
$(\langle f,\mu_t^N\rangle)_{N\ge 1}$ satisfies some regularity conditions, which is the purpose of the next result.

\begin{lemma}[Regularity condition]
\label{lem_reg_cond_inR}
Assume  {\rm \textbf{A1}}$\to${\rm \textbf{A4}}.
Then there exists $C>0$ such that a.s.  for all $f\in\mathcal C^\infty(\Theta)$, $0\leq r<t\leq 1$, and  $N\ge1$:
\begin{align}\label{bound_reg_cond}
|\langle f,\mu_t^N\rangle-\langle f,\mu_r^N\rangle|\leq C\Vert f\Vert_{\mathcal C^2(\Theta)}\Big[|t-r|+\frac{|t-r|}{N}+\frac 1N \Big].
\end{align}

\end{lemma}

\begin{proof}
Let $f\in\mathcal C^\infty(\Theta)$ and let $N\ge1$  and $0\leq r<t\leq 1$. In the following $C>0$ is a positive constant independent of  $f\in\mathcal C^\infty(\Theta)$,  $N\ge1$,  and $0\leq r<t\leq 1$, which can change from one occurrence to another.
From~\eqref{eq.pre_limit}, we have
\begin{align}
\nonumber
\langle f,\mu_t^N\rangle-\langle f,\mu_r^N\rangle&=\mathbf A_{r,t}^N[f]   - \eta \int_r^t \big \langle\nabla_\theta f\cdot \nabla_\theta \mathscr D_{{\rm KL}}(q_{\,_\cdot }^1|P_0^1),\mu_s^N\big \rangle\di s \\
\label{eq_prelim_reg_cond}
&\quad +\mathbf M_t^{N}[f]-\mathbf M_{r}^N[f] +\mathbf W_t^{N}[f]-\mathbf  W_{r}^N[f]+\mathbf R_t^N[f]-\mathbf R_r^N[f],
\end{align}
where
\begin{align*}
\mathbf A_{r,t}^N[f]&=-\eta\int_r^t\int_{\mathsf X\times\mathsf Y}\langle\phi(\cdot,\cdot,x)-y,\mu_s^N\otimes\gamma\rangle\langle\nabla_\theta f\cdot\nabla_\theta \phi(\cdot,\cdot,x),\mu_s^N\otimes\gamma\rangle  \pi(\di x,\di y)\nonumber\\
&\quad +\frac{\eta}{N}\int_r^t\int_{\mathsf X\times\mathsf Y} \Big\langle\langle\phi(\cdot,\cdot,x)-y,\gamma\rangle\langle\nabla_\theta f\cdot\nabla_\theta \phi(\cdot,\cdot,x),\gamma\rangle,\mu_s^N\Big\rangle \pi(\di x,\di y)\nonumber\\
&\quad -\frac{\eta}{N}\int_r^t\int_{\mathsf X\times\mathsf Y} \Big\langle(\phi(\cdot,\cdot,x)-y)\nabla_\theta f\cdot\nabla_\theta \phi(\cdot,\cdot,x),\mu_s^N\otimes\gamma\Big\rangle \pi(\di x,\di y).
\end{align*}
By \eqref{borne_phi-y} and \eqref{borne_nabla}, $
 |\mathbf A_{r,t}^N[f]| \leq C\Vert f\Vert_{\mathcal C^1(\Theta)}\big[|t-r|+\frac{|t-r|}{N}\big]$.
In addition, since $\theta\mapsto  \mathscr D_{{\rm KL}}(q_{\theta }^1|P_0^1)$ is bounded over $\Theta$ (since it is smooth and  $\Theta$  is compact),
$$
\Big| \int_r^t \big \langle\nabla_\theta f\cdot \nabla_\theta \mathscr D_{{\rm KL}}(q_{\,_\cdot }^1|P_0^1),\mu_s^N\big \rangle\di s  \Big|\leq C\Vert f\Vert_{\mathcal C^1(\Theta)}|t-r|.
$$
Furthermore, using \eqref{bound M_k},
$$
|\mathbf M_t^{N}[f]-\mathbf M_{r}^N[f]|=\Big|\sum_{k=\lfloor Nr\rfloor}^{\lfloor Nt\rfloor-1}\mathbf  M_{k}^{N}[f]\Big|\leq (\lfloor Nt\rfloor-\lfloor Nr\rfloor) {C\Vert f\Vert_{\mathcal C^1(\Theta)}}/{N}.
$$
Next, we have, by Item \ref{remTi6} in  Lemma \ref{le.error}, $
|\mathbf W_t^{N}[f]-\mathbf W_{r}^N[f]|\leq|\mathbf W_t^{N}[f]|+|\mathbf W_{r}^N[f]|\leq  {C}\Vert f\Vert_{\mathcal C^2(\Theta)}/{N}$.
Finally, by \eqref{boundR_k},
$$
|\mathbf R_t^N[f]-\mathbf  R_r^N[f]|=\Big|\sum_{k=\lfloor Nr\rfloor}^{\lfloor Nt\rfloor-1}\mathbf R_k^N[f]\Big|\leq (\lfloor Nt\rfloor-\lfloor Nr\rfloor) {C\Vert f\Vert_{\mathcal C^2(\Theta)}}/{N^2}.
$$
The proof of Proposition~\ref{lem_reg_cond_inR} is complete plugging all the previous estimates in \eqref{eq_prelim_reg_cond}.
\end{proof}

\begin{proposition}[Relative compactness]\label{p-rc}
Assume  {\rm \textbf{A1}}$\to${\rm \textbf{A4}}. Then, the sequence $(\mu^N)_{N\ge 1}$ is relatively compact in  $\mathcal D([0,1],\mathcal P(\Theta))$.
\end{proposition}

\begin{proof}
The proof consists in applying \cite[Theorem 3.1]{jakubowski1986skorokhod} with  $E=\mathcal P(\Theta)$  endowed with the weak convergence topology.
 Set $\mathbb F=\{\mathfrak L_f, f\in \mathcal C^\infty(\Theta)\}$ where
$$\mathsf  L_f: \nu \in \mathcal P(\Theta)\mapsto \langle f, \nu \rangle.$$

The class of continuous functions $\mathbb F$ on $\mathcal P(\Theta)$ satisfies Conditions~\cite[(3.1) and (3.2) in Theorem 3.1]{jakubowski1986skorokhod}.

On the other hand, the condition~\cite[(3.3) in Theorem 3.1]{jakubowski1986skorokhod} is satisfied  since $\mathcal P(\Theta)$ is compact because $\Theta$ is compact  (see e.g. \cite[Corollary 2.2.5]{panaretos2020invitation} together with \cite[Corollary 6.13]{villani2009optimal}).

It remains to verify Condition (3.4) of \cite[Theorem 3.1]{jakubowski1986skorokhod}, i.e.  that for all $f\in \mathcal C^\infty(\Theta)$, $(\langle f,\mu^N\rangle)_{N\ge1}$ is   relatively compact in $\mathcal D([0,1],\mathbf R)$. To this end, we apply \cite[Theorem~13.2]{billingsley2013convergence}. Condition (i) in \cite[Theorem~13.2]{billingsley2013convergence} is satisfied because
$|\langle f,\mu^N_t\rangle|\le \Vert f\Vert_{\infty,\Theta}$ for all $t\in [0,1]$ and $N\ge 1$.  Let us now show that Condition (ii) in \cite[Theorem~13.2]{billingsley2013convergence} holds.
For this purpose, we use Lemma \ref{lem_reg_cond_inR}.
For $\delta,\beta>0$  sufficiently small, it is possible to construct a subdivision $\lbrace t_i\rbrace_{i=0}^v$ of $[0,1]$ such that $t_0 =0$, $t_v=1$, $t_{i+1}-t_i = \delta+\beta$ for $i\in\{0,\dots,v-2\}$ and $\delta+\beta\leq t_v -t_{v-1} \leq 2(\delta+\beta)$.  According to the terminology introduced  in \cite[Section 12]{billingsley2013convergence},  $\lbrace t_i\rbrace_{i=0}^v$ is  $\delta$-sparse.  Then, by Lemma~\ref{lem_reg_cond_inR}, there exists $C>0$ such that  a.s.
  for all $\delta,\beta>0$, all such subdivision $\lbrace t_i\rbrace_{i=0}^v$,  $i\in\{0,\dots,v-1\}$, and $N\ge 1$,
$$\sup_{t,r\in[t_i ,t_{i+1} ]} \!\!\! |\langle f,\mu_t^N\rangle-\langle f,\mu_r^N\rangle|\leq C\Big(|t_{i+1} -t_{i} |+\frac{|t_{i+1} -t_{i} |}{N}+\frac 1N \Big)\leq C\Big(2(\delta+\beta)+\frac{2(\delta+\beta)}{N}+\frac 1N \Big).  $$
Thus,  one has:
$$\inf_{\beta>0}\max_i\sup_{t,r\in[t_i ,t_{i+1} ]}|\langle f,\mu_t^N\rangle-\langle f,\mu_r^N\rangle|\leq C\Big(2\delta+\frac{2\delta}{N}+\frac 1N \Big).$$
Consequently, there exists $C>0$ such that a.s. for all $\delta>0$ small enough and $N\ge 1$,
$$w'_{\langle f,\mu^N\rangle }(\delta):=\inf_{\substack{\{t_i\}\\ \delta\text{-sparse}}}\max_i\sup_{t,r\in[t_i,t_{i+1}]}|\langle f,\mu_t^N\rangle-\langle f,\mu_r^N\rangle|\leq C\Big(2\delta+\frac{2\delta}{N}+\frac 1N \Big). $$
This implies  $\lim_{\delta\to0}\limsup_{N\to+\infty}\mathbf E[w'_{\langle f,\mu^N\rangle }(\delta)]=0$. By Markov's inequality, this proves Condition~(ii) of \cite[Theorem~13.2]{billingsley2013convergence}. Therefore, for all $f\in\mathcal C^\infty(\Theta)$, using also Prokhorov theorem, the sequence $(\langle f,\mu^N\rangle)_{N\ge1}\subset \mathcal D([0,1],\mathbf R)$ is relatively compact.   In conclusion,
according to \cite[Theorem 3.1]{jakubowski1986skorokhod},  $(\mu^N)_{N\ge 1}\subset \mathcal D([0,1],\mathcal P(\Theta))$ is tight.
\end{proof}


\subsubsection{Limit points satisfy the limit equation \eqref{eq_limit}}
\label{sec.LP-1}

In this section we prove that every limit point of $(\mu^N)_{N\ge 1}$ in $\mathcal D([0,1],\mathcal P(\Theta))$ satisfies \eqref{eq_limit}.

\begin{lemma}\label{lem continuity proj sko borne}
Let $\mathsf m,(\mathsf m^N)_{N\ge 1}\subset \mathcal D([0,1],\mathcal P(\Theta))$  be such that $\mathsf m^N\to \mathsf m$ in $\mathcal D([0,1],\mathcal P(\Theta))$. Then, for all Lipschitz continuous function $f:\Theta\to\mathbf R$, we have $\langle f,\mathsf m^N\rangle\to \langle f,\mathsf m\rangle$ in $\mathcal D([0,1],\mathbf R)$.
\end{lemma}
\begin{proof}
Let $f$ be such a function.
By \cite[p.124]{billingsley2013convergence}, $\mathsf m^N\to \mathsf m$ in $\mathcal D([0,1],\mathcal P(\Theta))$ iff there exist functions $\lambda_N: [0,1]\to [0,1]$ continuous, increasing onto itself such that $\sup_{t\in[0,1]}|\lambda_N(t)-t|\to_{N\to \infty} 0$ and
$\sup_{t\in [0,1]}\mathsf W_1(\mathsf m_{\lambda_N(t)}^N,\mathsf m_t)\to_{N\to \infty}0$.
Then  $\langle f,\mathsf m^N\rangle \to \langle f,\mathsf m\rangle$ in $\mathcal D([0,1],\mathbf R)$ since by  \eqref{Kantorovitch Rubinstein}, $\sup_{t\in [0,1]}|\langle f,\mathsf m_{\lambda_N(t)}^N\rangle-\langle f,\mathsf m_t\rangle| \leq \|f\|_{\rm{Lip}}  \sup_{t\in [0,1]}\mathsf W_1(\mathsf m_{\lambda_N(t)}^N,\mathsf m_t)\to_{N\to \infty}0$.
\end{proof}

\begin{proposition}[Continuity of the limit points of $\langle f,\mu^N\rangle$] \label{prop lim point continuous fixe f}
Let $f\in \mathcal C^\infty(\Theta).$ Then, any limit point of $(\langle f,\mu^N\rangle)_{N\ge1}\subset \mathcal D([0,1],\mathbf R)$ belong a.s. to $\mathcal C([0,1],\mathbf R)$.
\end{proposition}
\begin{proof}
Fix $t\in (0,1]$. Letting  $r\to t$ in  \eqref{bound_reg_cond}, we obtain
 $ |\langle f,\mu_t^N\rangle-\langle f,\mu_{t^-}^N\rangle|\leq C/N$.
Therefore $\sup_{t\in(0,1]}|\langle f,\mu_t^N\rangle-\langle f,\mu_{t^-}^N\rangle| \xrightarrow{\mathscr D} 0$ as $N\to+\infty$. The result follows from  \cite[Theorem~13.4]{billingsley2013convergence}.
\end{proof}

\begin{proposition}[Continuity of the limit points of $\mu^N$]\label{prop_continuity_limit}
Let $\mu^*\in \mathcal D([0,1], \mathcal P(\Theta))$ be a limit point  of $(\mu^N)_{N\ge1}\subset \mathcal D([0,1], \mathcal P(\Theta))$. Then, a.s.   $\mu^*\in\mathcal C([0,1], \mathcal P(\Theta))$.
\end{proposition}
\begin{proof}
Up  to extracting a subsequence, we assume that $\mu^N\xrightarrow{\mathscr D} \mu^*$. By Skorohod representation theorem, there exists another probability space $(\hat{\Omega}, \hat{\mathcal F},\hat{\mathbf P})$ on which are defined random elements $(\hat\mu^N)_{N\ge1}$ and $\hat\mu^*$, where,
$$  \hat\mu^*\overset{\mathscr D}{=}\mu^*, \ \text{ and for all $N\ge1$, }\hat\mu^N\overset{\mathscr D}{=}\mu^N,$$
 and  such that $\hat{\mathbf P}$-a.s., $\hat\mu^N\to\hat\mu^*$ in  $\mathcal D([0,1], \mathcal P(\Theta))$ as $N\to +\infty$. Fix $f\in\mathcal C^\infty(\Theta)$. We have, by Lemma \ref{lem continuity proj sko borne},
\begin{equation*}
\hat{\mathbf P}{\rm{-a.s.}},\ \  \langle f,\hat\mu^N\rangle\to_{N\to+\infty}\langle f,\hat\mu^*\rangle\ \ \text{in} \ \ \mathcal D([0,1],\mathbf R).
\end{equation*}
In particular, $\langle f,\hat\mu^N\rangle\to_{N\to+\infty}\langle f,\hat\mu^*\rangle$ in distribution. By Proposition~\ref{prop lim point continuous fixe f}, there exists $\hat\Omega_f \subset \hat\Omega$ of $\hat{\mathbf P}$-mass 1 such that for all $\omega\in \hat\Omega_f,\  \langle f,\hat\mu^*(\omega)\rangle\in \mathcal C([0,1],\mathbf R)$. Denote by $\mathscr F$ the class polynomial functions with rational coefficients. Since this class is countable, the set $\hat\Omega_{\mathscr F}:=\cap_{f\in\mathscr F}\hat\Omega_f$
is of $\hat{\mathbf P}$-mass 1.
Consider now an arbitrary $f\in\mathcal C(\Theta)$ and let us show that for all $\omega\in \hat\Omega_{\mathscr F}, \ \langle f,\hat\mu^*(\omega)\rangle\in \mathcal C([0,1],\mathbf R)$. By the Stone-Weierstrass theorem, there exist $(f_n)_{n\ge1}\subset \mathscr F$ such that $\|f_n-f\|_{\infty,\Theta} \to_{n\to+\infty}0$. On $\hat\Omega_{\mathscr F}$, for all $n$,
$t\in [0,1]\mapsto \langle f_n,\hat\mu_t^*\rangle$ is continuous and converges uniformly to $t\in [0,1]\mapsto \langle f,\hat\mu_t^*\rangle$. 
Hence, for all $\omega\in \hat\Omega_{\mathscr F}$ and $f\in\mathcal C (\Theta)$, $\langle f,\hat\mu^*(\omega)\rangle \in \mathcal C([0,1],\mathbf R)$, i.e. for all $\omega\in \hat\Omega_{\mathscr F}$, $\hat\mu^*(\omega)\in \mathcal{C}([0,1],\mathcal P(\Theta))$. This concludes the proof.
\end{proof}
Now, we introduce, for $t\in[0,1]$ and $f\in \mathcal C^\infty(\Theta)$, the function $\boldsymbol \Lambda_t[f]:\mathcal D([0,1],\mathcal P(\Theta))\to \mathbf R_+$ defined by:
\begin{align}
\boldsymbol  \Lambda_t[f]:\mathsf m\mapsto &\Big|\langle f,\mathsf m_t\rangle-\langle f,\mu_0\rangle \nonumber\\
& +\eta\int_{0}^t\int_{\mathsf X\times\mathsf Y}\langle\phi(\cdot,\cdot,x)-y,\mathsf m_s\otimes\gamma\rangle\langle\nabla_\theta f\cdot\nabla_\theta \phi(\cdot,\cdot,x),\mathsf m_s\otimes\gamma\rangle  \pi(\di x,\di y)\di s\nonumber\\
\label{d-lambda}
&+  \eta\int_0^t \big \langle\nabla_\theta f\cdot \nabla_\theta \mathscr D_{{\rm KL}}(q_{\,_\cdot }^1|P_0^1),\mathsf m_s\big \rangle\di s  \Big|.
\end{align}

We now study the continuity of  $\boldsymbol  \Lambda_t[f]$.

\begin{lemma} \label{lem_cont_points}
Let $(\mathsf m^N)_{N\ge 1}\subset \mathcal D([0,1],\mathcal P(\Theta))$ converge to $\mathsf m\in\mathcal D([0,1],\mathcal P(\Theta))$. Then, for all continuity point $t\in[0,1]$ of $\mathsf m$ and all $f\in\mathcal C^\infty(\Theta)$, we have $\boldsymbol \Lambda_t[f](\mathsf m^N)\to \boldsymbol \Lambda_t[f](\mathsf m)$.
\end{lemma}
\begin{proof}
Let $f\in\mathcal C^\infty(\Theta)$ and denote by $\mathcal C(\mathsf m)\subset[0,1]$ the set of continuity points of $\mathsf m$. Let $t\in\mathcal C(\mathsf m)$. From \cite[p. 124]{billingsley2013convergence}, we have, for all $s\in \mathcal C(\mathsf m)$,
\begin{equation}\label{conv_fixed t}
\mathsf m^N_s\to\mathsf m_s\ \ \text{in}\ \ \mathcal P(\Theta).
\end{equation}
Thus, $\langle f,\mathsf m_t^N\rangle\to_{N\to\infty}\langle f,\mathsf m_t\rangle.$
For all $z\in\mathbf R^d$ and $(x,y)\in\mathsf X\times\mathsf Y$,  {\rm \textbf{A1}} and {\rm \textbf{A3}} ensure that the functions $\theta\in\Theta\mapsto\phi(\theta
,z,x)-y$ and $\theta\in\Theta\mapsto \nabla_\theta  f(\theta)\cdot\nabla_\theta \phi(\theta,z,x)$ are continuous and also bounded because $\Theta$ is compact. Hence, for all $s\in [0,t]\cap \mathcal C(\mathsf m)$, using \eqref{conv_fixed t},
\begin{align*}
\langle \phi(\cdot,z,x)-y,\mathsf m_s^N\rangle\to\langle \phi(\cdot,z,x)-y,\mathsf m_s\rangle\ \ \text{and}\ \  \langle\nabla_\theta f\cdot\nabla_\theta \phi(\cdot,z,x),\mathsf m_s^N\rangle\to \langle\nabla_\theta f\cdot\nabla_\theta \phi(\cdot,z,x),\mathsf m_s\rangle
\end{align*}
Since $[0,1]\backslash \mathcal C(\mathsf m)$ is at most countable (see \cite[p. 124]{billingsley2013convergence}) we have that for a.e. $(s,z',z,x,y)\in [0,t]\times\mathbf R^d\times \mathbf R^d\times\mathsf X\times\mathsf Y$,
\begin{equation*}
\langle \phi(\cdot,z',x)-y,\mathsf m_s^N\rangle\langle\nabla_\theta f\cdot\nabla_\theta \phi(\cdot,z,x),\mathsf m_s^N\rangle\to \langle \phi(\cdot,z',x)-y,\mathsf m_s\rangle\langle\nabla_\theta  f\cdot\nabla_\theta  \phi(\cdot,z,x),\mathsf m_s\rangle.
\end{equation*}
Since $\phi(\theta,z',x)-y$ is bounded and by \eqref{eq.Bntheta},  there exists $C>0$ such that for all $(s,z',z,x,y)\in [0,t]\times\mathbf R^d\times \mathbf R^d\times\mathsf X\times\mathsf Y$, $
 \langle |\phi(\cdot,z',x)-y|,\mathsf m_s^N\rangle\langle|\nabla_\theta f\cdot\nabla_\theta \phi(\cdot,z,x)|,\mathsf m_s^N\rangle\leq C\Vert \nabla _\theta f\Vert_{\infty,\Theta}\mathfrak b(z)$.
By the dominated convergence theorem, we then have:
\begin{align*}
\int_{0}^t\int_{\mathsf X\times\mathsf Y}\langle\phi(\cdot,\cdot,x)-y,\mathsf m_s^N\otimes\gamma\rangle\langle\nabla_\theta f\cdot\nabla_\theta \phi(\cdot,\cdot,x),\mathsf m_s^N\otimes\gamma\rangle  \pi(\di x,\di y)\di s \\
\underset{N\to+\infty}{\longrightarrow} \int_{0}^t\int_{\mathsf X\times\mathsf Y}\langle\phi(\cdot,\cdot,x)-y,\mathsf m_s\otimes\gamma\rangle\langle\nabla_\theta f\cdot\nabla_\theta \phi(\cdot,\cdot,x),\mathsf m_s\otimes\gamma\rangle  \pi(\di x,\di y)\di s.
\end{align*}
With the same arguments as above, one shows that $\int_0^t  \langle\nabla_\theta f\cdot \nabla_\theta \mathscr D_{{\rm KL}}(q_{\,_\cdot }^1|P_0^1),\mathsf m_s^N \rangle\di s \to \int_0^t  \langle\nabla_\theta f\cdot \nabla_\theta \mathscr D_{{\rm KL}}(q_{\,_\cdot }^1|P_0^1),\mathsf m_s \rangle\di s $.
 The proof of the lemma is complete.
\end{proof}

\begin{proposition}[Convergence to the limit equation]\label{p-cLE}
Let $\mu^*\in \mathcal D([0,1],\mathcal P(\Theta))$ be a limit point of $(\mu^N)_{N\geq1}\subset\mathcal D([0,1],\mathcal P(\Theta))$. Then, a.s.  $\mu^*$ satisfies \eqref{eq_limit}.
\end{proposition}
\begin{proof}
Up  to extracting a subsequence, we can assume that  $\mu^N\xrightarrow{\mathscr D} \mu^*$ as $N\to +\infty$. Let   $f\in \mathcal C^\infty(\Theta)$. The pre-limit equation \eqref{eq.pre_limit} and Lemma \ref{le.error} imply that   a.s. for all $N\ge 1$ and $t\in[0,1]$, $\boldsymbol \Lambda_t[f](\mu^N)\leq  C/N+ \mathbf M_t^{N}[f]$.
Hence, using the last statement in Lemma \ref{le.error}, it holds for all $t\in[0,1]$,
\begin{equation*}
\lim_{N\to\infty}\mathbf E[\boldsymbol  \Lambda_t[f](\mu^N)]=0.
\end{equation*}
In particular,  $\boldsymbol \Lambda_t[f](\mu^N)\xrightarrow{\mathscr D} 0$. Let us now show that $\boldsymbol \Lambda_t[f](\mu^N)\xrightarrow{\mathscr D}\boldsymbol \Lambda_t[f](\mu^*)$.
Denoting by $\mathsf D(\boldsymbol \Lambda_t[f])$ the set of discontinuity points of $\boldsymbol \Lambda_t[f]$, we have, from Proposition~\ref{prop_continuity_limit} and  Lemma \ref{lem_cont_points}, for all $t\in[0,1]$ and $f\in\mathcal C^\infty(\Theta)$,
\begin{equation*}
\mathbf P(\mu^*\in \mathsf D(\boldsymbol \Lambda_t[f])) =0.
\end{equation*}
By the continuous mapping theorem, $\boldsymbol \Lambda_t[f](\mu^N)\xrightarrow{\mathscr D}\boldsymbol \Lambda_t[f](\mu^*)$.
By uniqueness of the limit in distribution, we have that for all $t\in[0,1]$ and $f\in \mathcal C^\infty(\Theta)$,  a.s.  $\boldsymbol \Lambda_t[f](\mu^*)=0$. Let us now prove  that a.s. for all $t\in[0,1]$ and $f\in\mathcal C^\infty(\Theta)$, $\Lambda_t[f](\mu^*)=0$.

On the one hand, for all $f\in \mathcal C^\infty(\Theta)$ and $\mathsf m\in \mathcal D([0,1],\mathcal P(\Theta))$, the function $t\mapsto \Lambda_t[f](\mathsf m)$ is right-continuous. Since $[0,1]$ is separable, we have that for all $f\in \mathcal C^\infty(\Theta)$, a.s. for all $t\in[0,1]$, $\boldsymbol \Lambda_t[f](\mu^*)=0$.
\begin{sloppypar}
One the other hand  $ \mathcal C^\infty(\Theta)$ is separable (when endowed with the norm $\|f\|_{\mathcal C^\infty(\Theta)}= \sum_{k\ge 0}2^{-k}\min(1,\sum_{|j|=k}\|\partial_jf\|_{\infty,\Theta})$) and   the function $f\in \mathcal C^\infty(\Theta) \mapsto \boldsymbol \Lambda_t[f](\mathsf m)$ is continuous (for fixed $t\in[0,1]$ and $\mathsf m\in\mathcal{D}([0,1],\mathcal P(\Theta))$) relatively to the topology induced by $\|f\|_{\mathcal C^\infty(\Theta)}$. \end{sloppypar}

Hence, we obtain that  a.s. for all $t\in[0,1]$ and $f\in\mathcal C^\infty(\Theta)$, $\boldsymbol \Lambda_t[f](\mu^*)=0$. The proof of the proposition is thus complete.
\end{proof}



\subsubsection{Uniqueness and end of the proof of Theorem \ref{thm.ideal}}
\label{sec.U-1}


\begin{proposition}\label{p-uniq}
There exists a unique solution to \eqref{eq_limit} in $\mathcal C([0,1],\mathcal P(\Theta))$.
\end{proposition}
\begin{proof}
First of all, the fact that there is a solution to \eqref{eq_limit} is provided by  Propositions~\ref{p-rc},~\ref{prop_continuity_limit} and~\ref{p-cLE}.
The proof of the fact that there is a unique solution to  \eqref{eq_limit}   relies on  the same arguments as those used in the proof of~\cite[Proposition 2.14]{descours2022law}.

For $\mu\in\mathcal P(\mathbf R^{d+1})$, we introduce  $\boldsymbol v[\mu]:\mathbf R^{d+1}\to\mathbf R^{d+1}$ defined, for $\theta=(m,\rho)\in\mathbf R^{d+1}$,  by
\begin{align}
\label{def V[mu]}
&\boldsymbol v[\mu](\theta)=
-\eta\int_{\mathsf X\times\mathsf Y}\langle\phi(\cdot,\cdot,x)-y,\mu\otimes\gamma\rangle\langle\nabla_\theta\phi(\theta,\cdot,x),\gamma\rangle\pi(\di x,\di y)-\eta \nabla_\theta \mathscr D_{{\rm KL}}(q_{\theta}^1|P_0^1).
\end{align}
 In addition, if   $\bar \mu\in \mathcal C([0,1],\mathcal P(\Theta))$ is solution to  \eqref{eq_limit}, it satisfies also \eqref{eq_limit} with test functions $f\in\mathcal C^\infty_c( \mathbf R^{d+1})$. Then, adopting the terminology of \cite[Section
4.1.2]{santambrogio2015optimal},
any solution $\bar \mu$ to \eqref{eq_limit} is  a \textit{weak solution}\footnote{We mention that
  according to \cite[Proposition 4.2]{santambrogio2015optimal}, the
  two notions of solutions of~\eqref{eq.measure} (namely the weak
  solution and the \textit{distributional} solution) are equivalent.}
 on $[0,T]$ of the measure-valued equation
\begin{align}\label{eq.measure}
\begin{cases}
\partial_t\bar\mu_t=\mathrm{div}(\boldsymbol v[\bar\mu_t]\bar\mu_t)\\
\bar\mu_0=\mu_0.
\end{cases}
\end{align}
Let us now prove that:
\begin{enumerate}
\item There exists $C>0$ such that for all $\mu\in\mathcal P(\mathbf R^{d+1})$ and $\theta\in\mathbf R^{d+1}$, $$|\mathrm{J}_\theta \boldsymbol v[\mu](\theta)|\leq C. $$
\item There exists $C>0$ such that for all $\bar\mu\in\mathcal C([0,1],\mathcal P(\Theta))$ solution to \eqref{eq_limit},  $0\le s,t\le 1$, and   $\theta\in\mathbf R^{d+1}$,
$$| \boldsymbol v[\bar\mu_t](\theta)- \boldsymbol v[\bar\mu_s](\theta)|\leq C|t-s|. $$
\item There exists $L'>0$ such that for all $\mu,\nu\in \mathcal P_1(\mathbf R^{d+1})$,
$$\sup_{\theta\in\mathbf R^d}| \boldsymbol v[\mu](\theta)- \boldsymbol v[\nu](\theta)|\leq L'\mathsf W_1(\mu,\nu).$$
\end{enumerate}
\begin{sloppypar}

Before proving the three items above, we quickly conclude the proof of the proposition.  Items~1 and~2 above imply that $v(t,\theta)= \boldsymbol v[\bar\mu_t](\theta)$ is globally Lipschitz continuous over $[0,1]\times \mathbf R^{d+1}$ when $\bar\mu\in\mathcal C([0,1],\mathcal P(\Theta))$ is a solution to \eqref{eq_limit}. Since $\bar \mu \in \mathcal C([0,1],\mathcal P(\Theta))\subset \mathcal C([0,1],\mathcal P(\mathbf R^{d+1}))$, this allows to use the representation theorem~\cite[Theorem 5.34]{villani2021topics} for the solution  of~\eqref{eq.measure} in $\mathcal C([0,1],\mathcal P(\mathbf R^{d+1}))$, i.e. it holds:
\begin{equation}\label{eq.vi}
\forall t\in [0,1], \ \bar \mu_t=\phi_t\# \mu_0,
\end{equation}
 where $\phi_t$ is the flow generated by the vector field $\boldsymbol v[\bar \mu_t](\theta)$ over $\mathbf R^{d+1}$.
Equation \eqref{eq.vi} and the fact that $\mathcal C([0,1],\mathcal P(\Theta))\subset \mathcal C([0,1],\mathcal P_1(\mathbf R^{d+1}))$  together with  Item 3 above and the  same arguments as those used in the proof of~\cite[Proposition 2.14]{descours2022law} (which we recall is based estimates in  Wasserstein distances between two solutions of~\eqref{eq_limit} derived in \cite{piccoli2016properties}), one deduces that  there is a unique solution to  \eqref{eq_limit}.

Let us prove Item 1.
Recall $g(\rho)= \ln(1+e^{\rho})$. The functions
$$ \rho  \mapsto g''(\rho)g(\rho), \ \rho  \mapsto  g'(\rho),\  \rho  \mapsto  \frac{g'(\rho)}{g(\rho)}, \, \text{ and } \rho  \mapsto  \frac{g''(\rho)}{g(\rho)}$$
are bounded on $\mathbf R$. Thus, in view of \eqref{eq.kl_1}, $\Vert {\rm Hess}_\theta\, \mathscr D_{{\rm KL}}(q_{\theta }^1|P_0^1)\Vert_{\infty,\mathbf R^{d+1}}<+\infty$.
On the other hand,  by {\rm \textbf{A1}}  and  {\rm \textbf{A3}},  for $x\in\mathsf X$, $z\in\mathbf R^d$, $\theta \in \Theta \mapsto  \phi(\theta,z,x)$ is smooth and
 there exists $C>0$,  for all $x\in\mathsf X$, $\theta\in \mathbf R^{d+1}$,   $z\in\mathbf R^d$:
\begin{equation*}
| \mathrm{Hess}_\theta \phi(\theta,z,x) | \leq C(\mathfrak b(z)^2+\mathfrak b(z)).
\end{equation*}
This bound allows us to differentiate under the integral signs in \eqref{def V[mu]} and proves that $|\mathrm{J}_\theta \int_{\mathsf X\times\mathsf Y}\langle\phi(\cdot,\cdot,x)-y,\mu\otimes\gamma\rangle\langle\nabla_\theta\phi(\theta,\cdot,x),\gamma\rangle\pi(\di x,\di y)|\le C$, where $C>0$ is independent of $\mu\in\mathcal P(\Theta)$ and $\theta\in\Theta$.  The proof of Item~1 is complete.
\end{sloppypar}

Let us prove Item~2. Let $\bar\mu\in\mathcal C([0,1],\mathcal P(\Theta))$ be a  solution to \eqref{eq_limit},  $0\leq s\le t\le 1$, and  $\theta\in\mathbf R^{d+1}$. We have
\begin{equation}\label{syst_un}
 \boldsymbol v[\bar\mu_t](\theta)- \boldsymbol v[\bar\mu_s](\theta)=
-\eta\int_{\mathsf X\times\mathsf Y}\langle\phi(\cdot,\cdot,x),(\bar\mu_t-\bar\mu_s)\otimes\gamma\rangle\langle\nabla_\theta \phi(\theta,\cdot,x),\gamma\rangle\pi(\di x,\di y).
\end{equation}
Let $z\in \mathbf R^d$ and $x\in\mathsf X$. By {\rm \textbf{A1}} and {\rm \textbf{A3}}, $\phi(\cdot,z,x)\in\mathcal C^\infty(\Theta)$. Therefore, by \eqref{eq_limit},
\begin{align*}
\langle \phi(\cdot,z,x),\bar\mu_t-\bar\mu_s\rangle&= -\eta\int_{s}^t \!\! \int_{\mathsf X\times\mathsf Y}\!\! \!\! \langle\phi(\cdot,\cdot,x')-y,\bar\mu_r\otimes\gamma\rangle\langle\nabla_\theta\phi(\cdot,z,x)\cdot\nabla_\theta\phi(\cdot,\cdot,x'),\bar\mu_r\otimes\gamma\rangle  \pi(\di x',\di y)\di r\nonumber\\
&\quad-\eta\int_s^t\langle\nabla_\theta\phi(\cdot,z,x)\cdot \nabla_\theta \mathscr D_{{\rm KL}}(q_{\,_\cdot }^1|P_0^1),\bar\mu_r\rangle\di r
\end{align*}
We have  $ \Vert   \nabla_\theta \mathscr D_{{\rm KL}}(q_{\theta }^1|P_0^1)\Vert_{\infty,\Theta}<+\infty$. Using  also \eqref{eq.Bntheta}
and the fact that $\mathsf X\times\mathsf Y$ is a compact (see {\rm \textbf{A2}}), it holds:
$$|\langle \phi(\cdot,z,x),\bar\mu_t-\bar\mu_s\rangle|\leq C \mathfrak b(z)|t-s|.$$
Hence, for all $x'\in\mathsf X$,
\begin{equation*}
|\langle\phi(\cdot,\cdot,x'),(\bar\mu_t-\bar\mu_s)\otimes\gamma\rangle|\leq \langle|\langle\phi(\cdot,\cdot,x'),\bar\mu_t-\bar\mu_s\rangle|,\gamma\rangle\leq C|t-s|.
\end{equation*}
Thus, by \eqref{syst_un} and \eqref{borne nablaphi}, $| \boldsymbol v[\bar\mu_t](\theta)- \boldsymbol v[\bar\mu_s](\theta)|\leq C|t-s|$. This ends the proof of Item 2.

Let us now prove Item 3. Fix $\mu,\nu\in \mathcal P_1(\mathbf R^{d+1})$ and $\theta\in\mathbf R^{d+1}$.
We have
\begin{equation}\label{syst_un2}
 \boldsymbol v[\mu](\theta)- \boldsymbol v[\nu](\theta)= -\eta\int_{\mathsf X\times\mathsf Y}\langle\phi(\cdot,\cdot,x),( \mu -\nu)\otimes\gamma\rangle\langle\nabla_\theta \phi(\theta,\cdot,x),\gamma\rangle\pi(\di x,\di y)
\end{equation}
 For all $x\in\mathsf X$, using   \eqref{Kantorovitch Rubinstein} and  \eqref{eq.Bntheta}, it holds:
\begin{align*}
|\langle\phi(\cdot,\cdot,x),(\mu-\nu)\otimes\gamma\rangle|&\leq\int_{\mathbf R^d}|\langle\phi(\cdot,z,x),\mu\rangle-\langle\phi(\cdot,z,x),\nu\rangle|\gamma(z)\di z\\
&\leq C \int_{\mathbf R^d}\mathsf W_1(\mu,\nu)\mathfrak b(z)\gamma(z)\di z\leq C \mathsf W_1(\mu,\nu).
\end{align*}
Finally, using in addition \eqref{borne nablaphi} and \eqref{syst_un2}, we deduce Item~3.

  This ends the proof of the proposition.
\end{proof}

We are now ready to prove Theorem \ref{thm.ideal}.
\begin{proof}[Proof of Theorem \ref{thm.ideal}]
Recall Lemma \ref{lem:unif_bound_param} ensures that a.s. $(\mu^N)_{N\ge1}\subset \mathcal D([0,1],\mathcal P(\Theta))$. By Proposition~\ref{p-rc}, this sequence is relatively compact. Let $\mu^*\in \mathcal D([0,1],\mathcal P(\Theta))$ be a limit point. Along some subsequence $N'$, it holds:
$$\mu^{N'}\xrightarrow{\mathscr D}\mu^*.$$
In addition, a.s.  $\mu^*\in \mathcal C([0,1],\mathcal P(\Theta))$ (by Proposition~\ref{prop_continuity_limit}) and $\mu^*$ satisfies \eqref{eq_limit} (by Proposition~\ref{p-cLE}). By Proposition~\ref{p-uniq}, \eqref{eq_limit} admits a unique solution $\bar\mu\in\mathcal C([0,1],\mathcal P(\Theta))$. Hence, a.s. $\mu^*=\bar\mu$. Therefore,
$$\mu^{N'}\xrightarrow{\mathscr D}\bar\mu.$$
Since the sequence $(\mu^N)_{N\ge1}$ admits a unique limit point, the whole sequence converges in distribution to $\bar\mu$. The convergence also holds in probability since $\bar\mu$ is deterministic. The proof of Theorem \ref{thm.ideal} is complete.
\end{proof}


\subsection{Proof of Lemma \ref{lem:unif_bound_param}}
In this section we prove  Lemma \ref{lem:unif_bound_param}.
We start with the following simple result.


\begin{lemma}\label{lem_suite_exponentielle}
Let $T>0$, $N\ge 1$,    and $c_1>0$. Consider  a sequence $(u_k)_{0\leq k\leq \lfloor NT\rfloor } \subset \mathbf R_+$    for which  there exists $v_0$  such that $u_0\leq v_0$ and  for all  $1\leq k\leq  \lfloor NT\rfloor$, $u_k\leq c_1 (1+\frac 1N\sum_{\ell=0}^{k-1}u_\ell)$. Then, for all   $0\leq k\leq  \lfloor NT\rfloor$, $u_k\leq v_0e^{c_1T}$.
\end{lemma}
\begin{proof}
Define $v_k=c_1(1+\frac 1N\sum_{\ell=0}^{k-1}v_\ell)$. For all  $0\leq k\leq \lfloor NT\rfloor$, $u_k\le v_k$ and $v_k=v_{k-1}(1+c_1/N)$. Hence  $
v_k=v_0\big (1+ c_1/N\big)^k\leq v_{0}\big(1+ c_1/N\big)^{\lfloor NT\rfloor}\leq v^0e^{c_1T}$.
This ends the proof of the Lemma.
\end{proof}


\begin{proof}[Proof of Lemma \ref{lem:unif_bound_param}]
Since  $\rho\mapsto g'(\rho)$   and $\rho\mapsto g'(\rho)/g(\rho)$ are bounded continuous functions over $\mathbf R$, and since $|g(\rho)|\le C(1+|\rho|)$, according to \eqref{eq.kl_1}, there exists $c>0$, for all $\theta\in \mathbf R^{d+1}$,
\begin{equation}\label{eq.kl_2}
|\nabla_{\theta}\mathscr D_{{\rm KL}}(q_{\theta}^1|P_0^1)|\le c(1+|\theta|).
\end{equation}
 All along the proof, $C>0$ is a constant independent of $N\ge 1$,  $T>0$, $i\in \{1,\dots, N\}$, $1\leq k\leq \lfloor NT\rfloor $, $(x,y)\in \mathsf X\times \mathsf Y$, $\theta \in \mathbf R^{d+1}$, and $z\in \mathbf R^d$,  which can change from one occurence to another.
It holds:
\begin{align}\label{lem_moments_eq1}
|\theta_k^i|\leq |\theta_{0}^i|+ \sum_{\ell=0}^{k-1}|\theta_{\ell+1}^i-\theta_\ell^i|.
\end{align}
Using \eqref{eq.algo-ideal}, we have, for $0\leq \ell\leq k-1$,
\begin{align}\label{unif_param_eq1}
|\theta_{\ell+1}^i-\theta_\ell^i|&\leq \frac{\eta}{N^2}\sum_{j=1,j\neq i}^N\Big|(\langle\phi(\theta_{\ell}^j,\cdot,x_{\ell}),\gamma\rangle-y_{\ell})\langle\nabla_\theta\phi(\theta_{\ell}^i,\cdot,x_{\ell}),\gamma\rangle \Big|\nonumber\\
&\quad+ \frac{\eta}{N^2}\Big|\Big\langle(\phi(\theta_{\ell}^i,\cdot,x_{\ell})-y_{\ell})\nabla_\theta\phi(\theta_{\ell}^i,\cdot,x_{\ell}),\gamma\Big\rangle\Big|  +\frac{\eta}{N} |\nabla_{\theta}\mathscr D_{{\rm KL}}(q_{\theta_\ell^i}^1|P_0^1)|.
\end{align}
For all $\theta\in \mathbf R^{d+1}$, $z\in\mathbf R^d$, $(x,y)\in\mathsf X\times\mathsf Y$,  we have, by {\rm \textbf{A2}} and {\rm \textbf{A3}}, since $\phi(\theta,z,x)=s(\Psi_\theta(z),x)$,
\begin{align}\label{borne_phi-y}
|\phi(\theta,z,x)-y|\leq C.
\end{align}
Moreover,  we have $\nabla_\theta\phi(\theta,z,x)=\nabla_1s(\Psi_\theta(z),x) \mathrm J_\theta\Psi_\theta(z)$ (here $\nabla_1s$ refers to the gradient of $s$ w.r.t. its first variable). By {\rm \textbf{A3}}, $|\nabla_1s(\Psi_\theta(z),x)|\leq C$ and, hence,  denoting by $\mathrm J_\theta$  the Jacobian w.r.t.  $\theta$, using \eqref{jac_T_bounded},
\begin{equation}\label{eq.Bntheta}
|\nabla_\theta\phi(\theta,z,x)|\leq C|\mathrm J_\theta\Psi_\theta(z)|\leq   C\mathfrak b(z).
\end{equation}
Therefore, by \eqref{jac_T_bounded},
\begin{equation}\label{borne nablaphi}
\langle|\nabla_\theta\phi(\theta,\cdot,x)|,\gamma\rangle\leq C.
\end{equation}
Hence, we obtain, using \eqref{unif_param_eq1} and \eqref{eq.kl_2},
\begin{align}\label{m^l+1-m^l}
|\theta_{\ell+1}^i-\theta_\ell^i|&\leq \frac{\eta}{N^2}\sum_{j=1,j\neq i}^NC+\frac{\eta}{N^2}C + \frac{c\eta}{N}(1+|\theta_\ell^i|)  \leq \frac CN(1+ |\theta_\ell^i|).
\end{align}
Using {\rm \textbf{A4}}, there exists $K_0>0$ such that a.s. for all $i$, $|\theta_{0}^i|\leq K_0$.
Then, from \eqref{lem_moments_eq1} and  \eqref{m^l+1-m^l}, for $1\leq k\leq \lfloor NT\rfloor $, it holds:
\begin{align*}
|\theta_k^i|\leq  K_0 + \frac CN \sum_{\ell=0}^{k-1}(1+|\theta_\ell^i|)\le K_0+CT+ \frac CN \sum_{\ell=0}^{k-1} |\theta_\ell^i|\le C_{0,T}(1+ \frac 1N \sum_{\ell=0}^{k-1} |\theta_\ell^i|),
\end{align*}
with $C_{0,T}=\max(K_0+CT, C)\le K_0+C(1+T)$.
 Then, by Lemma \ref{lem_suite_exponentielle} and {\rm \textbf{A4}}, we have that for all $N\ge1$, $i\in\{1,\dots,N\}$ and $0\leq k\leq \lfloor NT\rfloor $, $|\theta_k^i|\leq K_0e^{[K_0+C(1+T)]T}$.
  The proof of Lemma~\ref{lem:unif_bound_param}  is thus complete.
\end{proof}

\section{Proof of Theorem \ref{thm.z1zN}}
\label{sec.proof2}

In this  section, we assume   {\rm \textbf{A1}}$\to$$\mathbf{A5}$
  (where in {\rm \textbf{A2}}, when $k\ge 1$, $\mathcal F_k^N$ is now the one defined in~\eqref{eq.Fk2}) and the $\theta^i_k$'s (resp. $\mu^N$)  are those  defined by~\eqref{eq.algo-batch} for  $i\in \{1,\ldots,N\}$ and $k\ge 0$ (resp. by~\eqref{empirical_distrib2} for $N\ge 1$).
\subsection{Preliminary analysis and pre-limit equation}

\subsubsection{Notation and weighted Sobolev embeddings}
\label{sec.n-s} For $J\in\textbf{N}$ and $\beta\geq0$, let  $\mathcal H^{J,\beta}(\mathbf R^{d+1})$ be the closure of the set
$\mathcal C_c^\infty(\mathbf R^{d+1})$ for the norm
$$\|f\|_{\mathcal H^{J,\beta}}:=\Big(\sum_{|k|\leq J}\int_{\mathbf R^{d+1}}\frac{|\partial_kf(\theta)|^2}{1+|\theta|^{2\beta}}\di \theta\Big)^{1/2}.$$
 The space
$\mathcal H^{J,\beta}(\mathbf R^{d+1})$ is a separable  Hilbert space and we  denote  its dual space by
$\mathcal H^{-J,\beta}(\mathbf R^{d+1})$ (see e.g.  \cite{fernandez1997hilbertian,jourdainAIHP}).
The associated
scalar product on $\mathcal H^{J,\beta}(\mathbf R^{d+1})$ will be denoted
by $\langle\cdot,\cdot\rangle_{\mathcal H^{J,\beta}}$. For
$\Phi \in \mathcal H^{-J,\beta}(\mathbf R^{d+1})$, we use the notation
$$\langle f,\Phi\rangle_{J,\beta}= \Phi[f], \ f\in \mathcal H^{J,\beta}(\mathbf R^{d+1}).$$
For ease of notation, and if no confusion is possible, we simply
denote $\langle f,\Phi\rangle_{J,\beta}$ by $\langle f,\Phi\rangle$.
The set $\mathcal C^{J,\beta}_0(\mathbf R^{d+1})$ (resp. $\mathcal C^{J,\beta}(\mathbf R^{d+1})$) is defined as the space of
functions $f:\mathbf R^{d+1}\rightarrow\mathbf{R}$ with continuous
partial derivatives up to order $J\in\textbf{N}$ such that
$$\text{for all} \ |k|\leq J, \ \lim_{|\theta|\rightarrow\infty}\frac{|\partial_kf(\theta)|}{1+|\theta|^\beta}=0\  \text{ (resp. $\sum_{|k|\leq J}\ \sup_{\theta\in\mathbf R^{d+1}}\frac{|\partial_kf(\theta)|}{1+|\theta|^\beta}<+\infty$)}.$$
The spaces $\mathcal C^{J,\beta}(\mathbf R^{d+1})$ and $\mathcal C^{J,\beta}_0(\mathbf R^{d+1})$  is endowed with the norm
$$\|f\|_{\mathcal C^{J,\beta}}:=\sum_{|k|\leq J}\ \sup_{\theta\in\mathbf R^{d+1}}\frac{|\partial_kf(\theta)|}{1+|\theta|^\beta}.$$
  We note that
\begin{equation}\label{eq.inclu-x}
\theta\in \mathbf{R}^{d+1} \mapsto (1-\chi(\theta))|\theta|^\alpha\in \mathcal
H^{J,\beta}(\mathbf{R}^{d+1})\text{ if } \beta-\alpha>(d+1)/2,
\end{equation}
 where $\chi\in \mathcal C_c^\infty(\mathbf R^{d+1})$ equals $1$ near $0$.
We recall that from
\cite[Section 2]{fernandez1997hilbertian}, for $m'>(d+1)/2$  and  $\alpha,j\ge 0$, $\mathcal H^{m'+j,\alpha}(\mathbf R^{d+1})\hookrightarrow \mathcal C_0^{j,\alpha}(\mathbf R^{d+1})$.
In the following, we consider $\gamma_0,\gamma_1\in \mathbf R$ and $L_0\in \mathbf N$ such that
$$\gamma_1>\gamma_0> \frac{d+1}2+1 \text{ and } L_0> \frac{d+1}2  +1.$$
We finally recall the following standard result.

\begin{proposition}\label{prop.compact_wasserstein}
Let $q>p\ge 1$ and $C>0$. The set $\mathscr K_C^{q}:=\{\mu\in \mathcal P_p(\mathbf R^{d+1}), \int_{\mathbf R^{d+1}}|x|^{q}\mu(\di x)\leq C\}$ is compact.
\end{proposition}

 \subsubsection{Bound on the moments of the $\theta_k^i$'s}
 \begin{sloppypar}
We have the following uniform bound in $N\ge 1$  on the moments  of the sequence $\{\theta_k^i, i\in \{1,\ldots,N\}\}_{k= 0,\ldots, \lfloor NT \rfloor}$ defined by \eqref{eq.algo-batch}.
 \end{sloppypar}

\begin{lemma}\label{lem:moment_param}
Assume  {\rm \textbf{A1}}$\to$ $\mathbf{A5}$. For all $T>0$ and $p\ge 1$, there exists $C>0$ such that for all $N\ge1$, $i\in\{1,\dots,N\}$ and $0\leq k\leq \lfloor NT\rfloor$,
$$\mathbf E[|\theta_k^i|^p]\leq C.$$
\end{lemma}
 \begin{proof}
 Let  $p\ge 1$.
By {\rm \textbf{A4}}, $\mathbf E[|\theta_0^i|^p]\le C_p$ for all $i\in\{1,\dots,N\}$.  Let  $T>0$.
In the following $C>0$ is a constant  independent of $N\ge1$, $i\in\{1,\dots,N\}$, and $1\leq k\leq \lfloor NT\rfloor$.
 Using \eqref{eq.algo-batch},  the fact that $\phi$ is bounded, $\mathsf Y$ is bounded, and \eqref{eq.Bntheta},   we have, for $0\leq  n \leq k-1$,
\begin{align}
\nonumber
|\theta_{n+1}^i-\theta_n^i|&\leq \frac{C}{N^2B}\sum_{j=1}^N\sum_{\ell=1}^B \mathfrak b(\mathsf Z^{i,\ell}_n) +\frac{C}{N} |\nabla_{\theta}\mathscr D_{{\rm KL}}(q_{\theta_n^i}^1|P_0^1)|\\
\label{eq.diff}
&\le  \frac{C}{NB}\sum_{\ell=1}^B (1+\mathfrak b(\mathsf Z^{i,\ell}_n)) +\frac{C}{N} (1+|\theta_n^i|),
\end{align}
where we have also used \eqref{eq.kl_2} for the last inequality.
Let us recall the following convexity inequality: for $m,p\geq 1$
  and $x_1,\dots,x_p\in\mathbf{R}_+$,
\begin{equation}\label{convexity inequality}
\Big(\sum_{n=1}^mx_n\Big)^p\leq m^{p-1}\sum_{n=1}^mx_n^p.
\end{equation}
Using \eqref{lem_moments_eq1}, {\rm \textbf{A1}} with $q=p$,  and the fact that $1\le k \leq \lfloor NT\rfloor$, one has setting $u_k=\mathbf E[|\theta_k^i|^p]$,  $ u_k\le C (1+\frac 1N\sum_{n=0}^{k-1}u_n)$. The result then follows from
 Lemma \ref{lem_suite_exponentielle}.
 \end{proof}


\subsubsection{Pre-limit equation}
In this section, we derive the pre-limit equation for $\mu^N$ defined by \eqref{empirical_distrib2}.  For simplicity    we will keep  the same notations   as those introduced  in Section~\ref{sec.PR1}, though these objects will now  be defined with $\theta^i_k$ set by \eqref{eq.algo-batch}, and  on $  \mathcal C^{2,\gamma_1}(\mathbf R^{d+1})$,   for all integer $k\ge 0$, and all time $t\ge 0$.  Let $f\in\mathcal C^{2,\gamma_1}(\mathbf R^{d+1})$.
Then, set for $k\ge 0$,
\begin{align*}
\mathbf D_{k}^N[f]&=-\frac{\eta}{N^3}\sum_{i=1}^N\sum_{j=1,j\neq i}^N\int_{\mathsf X\times\mathsf Y}\big (\big \langle\phi(\theta_k^j,\cdot,x),\gamma\big \rangle-y\big )\big \langle\nabla_\theta f(\theta_k^i)\cdot\nabla_\theta\phi(\theta_k^i,\cdot,x),\gamma\big \rangle\pi(\di x,\di y)\\
&\quad-\frac{\eta}{N^2}\int_{\mathsf X\times\mathsf Y}\big \langle(\phi(\cdot,\cdot,x)-y)\nabla_\theta f\cdot\nabla_\theta\phi(\cdot,\cdot,x),\nu_k^N\otimes\gamma\big \rangle\pi(\di x,\di y).
\end{align*}
Note that $\mathbf D_{k}^N$ above is  the one  defined in~\eqref{D_k-calcule} but now on   $\mathcal C^{2,\gamma_1}(\mathbf R^{d+1})$ and with $\theta^i_k$ defined by~\eqref{eq.algo-batch}.
For $k\ge 0$, we set
\begin{align}\label{def M_k1_z1,Z_N}
\mathbf M_{k}^N[f]= -\frac{\eta}{N^3B}\sum_{i,j=1}^N  \sum_{\ell=1}^B(\phi(\theta_k^j,\mathsf  Z_k^{j,\ell},x_k)-y_k)\nabla_\theta f(\theta_k^i)\cdot\nabla_\theta \phi(\theta_k^i,\mathsf Z_k^{i,\ell},x_k)-\mathbf D_{k}^N[f].
\end{align}
 By Lemma~\ref{lem:moment_param} together with  \eqref{borne_phi-y} and \eqref{eq.Bntheta}, $\mathbf M_{k}^{N}[f]$ is integrable.
Also, using   \textbf{A5} and the fact that $\theta_k^j$ is $\mathcal F_k^N$-measurable (see~\eqref{eq.Fk2}),
$$\mathbf E  [\mathbf M_{k}^{N}[f]|\mathcal F_k^N ]=0.$$
 Set $\mathbf  M_t^{N}[f]=\sum_{k=0}^{\lfloor Nt\rfloor-1}\mathbf M_{k}^{N}[f]$, $t\ge 0$.  We now extend the definition of $\mathbf W_t^N[f]$
and $\mathbf  R_k^N[f]$  in \eqref{def V_t^N} and \eqref{def_R_k^N} to any time $t\ge 0$, $k\ge 0$, and $f\in\mathcal C^{2,\gamma_1}(\mathbf R^{d+1})$, and with $\theta^i_k$ set by \eqref{eq.algo-batch}. We then set
$$\mathbf R_t^N[f]=\sum_{k=0}^{\lfloor Nt\rfloor-1} \mathbf  R_k^N[f], \ t\ge 0.$$
 With the same algebraic computations as those made  in Section~\ref{sec:pre_limit_eq}, one obtains the   following pre-limit equation: for $N\ge 1$, $t\ge 0$, and $f\in \mathcal C^{2,\gamma_1}(\mathbf R^{d+1})$,
\begin{align}\label{eq.pre_limitz1zN}
\langle f,\mu_t^N\rangle-\langle f,\mu_0^N\rangle&=-\eta\int_{0}^t\int_{\mathsf X\times\mathsf Y}\langle\phi(\cdot,\cdot,x)-y,\mu_s^N\otimes\gamma\rangle\langle\nabla_\theta f\cdot\nabla_\theta \phi(\cdot,\cdot,x),\mu_s^N\otimes\gamma\rangle  \pi(\di x,\di y)\di s\nonumber\\
&\quad -\eta \int_0^t \big \langle\nabla_\theta f\cdot \nabla_\theta \mathscr D_{{\rm KL}}(q_{\,_\cdot }^1|P_0^1),\mu_s^N\big \rangle\di s  \nonumber\\
&\quad +\frac{\eta}{N}\int_{0}^t\int_{\mathsf X\times\mathsf Y} \Big\langle\langle\phi(\cdot,\cdot,x)-y,\gamma\rangle\langle\nabla_\theta  f\cdot\nabla_\theta \phi(\cdot,\cdot,x),\gamma\rangle,\mu_s^N\Big\rangle \pi(\di x,\di y)\di s\nonumber\\
&\quad -\frac{\eta}{N}\int_{0}^t\int_{\mathsf X\times\mathsf Y} \Big\langle(\phi(\cdot,\cdot,x)-y)\nabla_\theta f\cdot\nabla_\theta \phi(\cdot,\cdot,x),\mu_s^N\otimes\gamma\Big\rangle \pi(\di x,\di y)\di s\nonumber\\
&\quad +  \mathbf M_t^{N}[f] +\mathbf W_t^{N}[f]+ \mathbf R_t^N[f].
\end{align}

We will now show that the sequence $(\mu^N)_{N\ge1}$ is relatively compact in $\mathcal D(\mathbf R_+,\mathcal P_{\gamma_0}(\mathbf R^{d+1}))$.

\subsection{Relative compactness and convergence to the limit equation}
\label{re.rc-Sobo}
\subsubsection{Relative compactness in $\mathcal D(\mathbf R_+,\mathcal P_{\gamma_0}(\mathbf R^{d+1}))$}

 In this section we prove the following result.

\begin{proposition}\label{prop_rc_inP}
Assume  {\rm \textbf{A1}}$\to$$\mathbf{A5}$. Recall $\gamma_0> \frac{d+1}2+1$.
Then, the sequence $(\mu^N)_{N\ge1}$ is relatively compact in $\mathcal D(\mathbf R_+,\mathcal P_{\gamma_0}(\mathbf R^{d+1}))$.
\end{proposition}

We start with the following lemma.

\begin{lemma}\label{lem_cc_sob}
Assume  {\rm \textbf{A1}}$\to$  $\mathbf{A5}$. Then,  $\forall T>0$ and $f\in \mathcal C^{2,\gamma_1}(\mathbf R^{d+1})$,
$$
\sup_{N\ge1}\mathbf E\Big[\sup_{t\in[0,T]}\langle   f,\mu_t^N\rangle^2\Big]<+\infty.
$$
\end{lemma}
\begin{proof}
Let $T>0$. In what follows, $C>0$ is a constant independent of $f\in \mathcal C^{2,\gamma_1}(\mathbf R^{d+1})$,  $(s,t)\in [0,T]^2$,  and $z\in \mathbf R^{d}$ which can change from one occurence to another. We have by {\rm \textbf{A4}},  $\mathbf E[\langle f,\mu_0^N\rangle^2]\leq C \|f\|_{\mathcal C^{2,\gamma_1}}^2$.
By \eqref{eq.pre_limitz1zN} and \eqref{borne_phi-y}, it holds:
\begin{align}\label{eq_cc_sob}
\sup_{t\in[0,T]} \langle f,\mu_t^N\rangle^2&\leq C\Big[  \|f\|_{\mathcal C^{2,\gamma_1}}^2+ \int_{0}^T\int_{\mathsf X\times\mathsf Y} \big |\Big\langle \langle \big |\nabla_\theta  f\cdot\nabla_\theta \phi(\cdot,\cdot,x) \big |,\gamma\rangle,\mu_s^N\Big\rangle\big | ^2  \pi(\di x,\di y)\di s\nonumber\\
&\quad   \int_0^ T \big |  \big \langle \big |\nabla_\theta f\cdot \nabla_\theta \mathscr D_{{\rm KL}}(q_{\,_\cdot }^1|P_0^1) \big |,\mu_s^N\big \rangle\big | ^2\di s  \nonumber\\
&\quad +\frac{1}{N^2}\int_{0}^T\int_{\mathsf X\times\mathsf Y}  \big |\Big\langle \langle \big |\nabla_\theta  f\cdot\nabla_\theta \phi(\cdot,\cdot,x) \big |,\gamma\rangle,\mu_s^N\Big\rangle\big | ^2 \pi(\di x,\di y)\di s\nonumber\\
&\quad +  \sup_{t\in[0,T]} |\mathbf M_t^{N}[f]|^2 +\sup_{t\in[0,T]} |\mathbf W_t^{N}[f]|^2+ \sup_{t\in[0,T]} |\mathbf R_t^N[f]|^2.\Big].
\end{align}
We have using  \eqref{eq.Bntheta},  for $s\in [0,T]$ and $z\in \mathbf R^{d}$,
\begin{equation}\label{eq.E1}
| \nabla_\theta f (\theta^i_{\lfloor Ns\rfloor}) \cdot\nabla_\theta \phi(\theta^i_{\lfloor Ns\rfloor},z,x)|\le C  \|f\|_{\mathcal C^{1,\gamma_1}} \mathfrak b(z) (1+|\theta^i_{\lfloor Ns\rfloor}|^{\gamma_1}).
\end{equation}
 Thus, using Lemma~\ref{lem:moment_param},
\begin{equation}\label{eq.E-3}
  \mathbf E\big[  \big\langle  \langle|\nabla_\theta  f\cdot\nabla_\theta \phi(\cdot,\cdot,x)|,\gamma\rangle ,\mu_s^N\big\rangle^2 \big]\leq C\|f\|_{\mathcal C^{1,\gamma_1}}^2.
\end{equation}
Using \eqref{eq.kl_2}, for $s\in [0,T]$, it holds:
\begin{equation}\label{eq.E2}
\big |  \nabla_\theta f(\theta^i_{\lfloor Ns\rfloor})\cdot \nabla_\theta \mathscr D_{{\rm KL}}(q_{\theta^i_{\lfloor Ns\rfloor} }^1|P_0^1) \big | \le C  \|f\|_{\mathcal C^{1,\gamma_1}} (1+|\theta^i_{\lfloor Ns\rfloor}|^{\gamma_1+1}).
\end{equation}
Thus, using Lemma~\ref{lem:moment_param},
\begin{equation}\label{eq.E-4}
\mathbf E\big [\big |  \big \langle\nabla_\theta f\cdot \nabla_\theta \mathscr D_{{\rm KL}}(q_{\,_\cdot }^1|P_0^1),\mu_s^N\big \rangle\big | ^2\big ]\leq C\|f\|_{\mathcal C^{1,\gamma_1}}^2.
\end{equation}
On the other hand, we have using \eqref{convexity inequality}:
\begin{align}\label{eq_m_t}
\sup_{t\in [0,T]}|\mathbf M_t^N[f]|^2\leq \lfloor NT\rfloor \sum_{k=0}^{\lfloor NT\rfloor-1}| \mathbf M_{k}^N[f]|^2.
\end{align}
Recall \eqref{def M_k1_z1,Z_N}. By \eqref{D_k-calcule},  \eqref{convexity inequality},   {\rm \textbf{A1}}, and \eqref{eq.E1}, it holds:
$$ |\mathbf D_{k}^N[f]|^2\le C\|f\|_{\mathcal C^{1,\gamma_1}}^2 \big[\frac{1}{N^4} \sum_{i\neq j=1}^N      (1+|\theta^i_{k}|^{2\gamma_1})+ \frac{1}{N^4}      (1+\langle |\cdot |^{2\gamma_1}, \nu_k^N\rangle)\big]\le \frac{C}{N^2}   \|f\|_{\mathcal C^{1,\gamma_1}}^2   (1+|\theta^i_{k}|^{2\gamma_1})$$
and
$$
|\mathbf M_{k}^N[f]|^2\leq \frac{C}{N^4B}\sum_{i,j=1}^N  \sum_{\ell=1}^B\|f\|^2_{\mathcal C^{1,\gamma_1}} |\mathfrak b(\mathsf Z_k^{i,\ell})|^2 (1+|\theta^i_{\lfloor Ns\rfloor}|^{2\gamma_1})+ |\mathbf D_{k}^N[f]|^2.
$$
By Lemma \ref{lem:moment_param} and {\rm \textbf{A1}}, one deduces  that
\begin{equation}\label{eq.M1}
\mathbf E[|\mathbf M_{k}^N[f]|^2]\leq {C\|f\|_{\mathcal C^{1,\gamma_1}}^2}/{N^2}.
\end{equation}
 Going back to \eqref{eq_m_t}, we then have $
\mathbf E[\sup_{t\in [0,T]}|\mathbf M_t^N[f]|^2]\leq C\|f\|_{\mathcal C^{1,\gamma_1}}^2$.
Using the same arguments as those used so far,
one also deduces that for $t\in [0,T]$
\begin{align*}
 \sup_{t\in[0,T]}|\mathbf W_t^N[f]|^2  &\leq \frac{C\|f\|_{\mathcal C^{1,\gamma_1}}^2}{N^2}   \sup_{t\in[0,T]} (1+\langle |\cdot |^{\gamma_1+1}, \nu_{\lfloor Nt\rfloor}^N\rangle)^2\\
 &= \frac{C\|f\|_{\mathcal C^{1,\gamma_1}}^2}{N^2}  \max_{0\leq k\leq \lfloor NT\rfloor}(1+\langle |\cdot |^{\gamma_1+1}, \nu_{k}^N\rangle)^2\\
 &\le \frac{C\|f\|_{\mathcal C^{1,\gamma_1}}^2}{N^2}  \sum_{k=0}^{\lfloor NT\rfloor} (1+\langle |\cdot |^{\gamma_1+1}, \nu_{k}^N\rangle)^2.
\end{align*}
and thus
\begin{equation}\label{eq.Wt}
\mathbf E\Big[\sup_{t\in[0,T]}|\mathbf W_t^N[f]|^2\Big] \leq  {C\|f\|_{\mathcal C^{1,\gamma_1}}^2}/{N}.
\end{equation}
 Let us finally deal with the term involving $\mathbf R_t^N[f]$.
One has  using \eqref{convexity inequality}:
$$\sup_{t\in[0,T]}|\mathbf R_t^N[f]|^2\leq \lfloor NT\rfloor\sum_{k=0}^{\lfloor NT\rfloor-1}|\mathbf R_k[f]|^2.$$
For $0\leq k\le \lfloor NT\rfloor-1$,  we have, from \eqref{def_R_k^N},
\begin{align*}
|\mathbf R_k^N[f]|^2&\leq\frac{C\|f\|_{\mathcal C^{2,\gamma_1}}^2}{N}\sum_{i=1}^N|\theta_{k+1}^i-\theta_k^i|^4(1+|\hat{\theta}_k^i|^{\gamma_1})^2\\
&\leq \frac{C\|f\|_{\mathcal C^{2,\gamma_1}}^2}{N}\sum_{i=1}^N|\theta_{k+1}^i-\theta_k^i|^4(1+|\theta_{k+1}^i|^{2\gamma_1}+|\theta_k^i|^{2\gamma_1}).
\end{align*}
Using \eqref{eq.diff},
\begin{align*}
|\theta_{k+1}^i-\theta_k^i|^4\leq C\Big[\frac{1}{N^4}+\frac{|\theta_k^i|^4}{N^4}+\frac{1}{N^4B}\sum_{\ell=1}^B|\mathfrak b(\mathsf Z_k^{i,\ell})|^4\Big].
\end{align*}
By Lemma \ref{lem:moment_param} and {\rm \textbf{A1}}, it then holds
$\mathbf E[|\theta_{k+1}^i-\theta_k^i|^4(1+|\theta_{k+1}^i|^{2\gamma_1}+|\theta_k^i|^{2\gamma_1})] \leq {C}/{N^4}$.
Hence, one deduces that
\begin{align}\label{eq.Rt}
\mathbf E[\sup_{t\in[0,T]}|\mathbf R_t^N[f]|^2]\leq C\| f\|_{\mathcal C^{2,\gamma_1}}^2 /N^2.
\end{align}
 This ends the proof of Lemma \ref{lem_cc_sob}.
\end{proof}

\begin{lemma}[Compact containment for $(\mu^N)_{N\ge1}$] \label{lem_cc_mu^N}
Assume  {\rm \textbf{A1}}$\to$$\mathbf{A5}$.  Let $0<\epsilon<\gamma_1-\gamma_0$. For every $T>0$,
\begin{equation}
\sup_{N\ge1}\mathbf E\Big[\sup_{t\in[0,T]}\int_{\mathbf R^{d+1}}|x|^{\gamma_0+\epsilon}\mu_t^N(\di x) \Big] <+\infty.
\end{equation}
\end{lemma}

\begin{proof}
Apply  Lemma \ref{lem_cc_sob}   with   $f:\theta\mapsto(1-\chi)|\theta|^{\gamma_0+\epsilon}\in \mathcal C^{2,\gamma_1}(\mathbf R^{d+1})$.
\end{proof}

\begin{lemma}\label{lem_rc_cont_sob}
Assume  {\rm \textbf{A1}}$\to$$\mathbf{A5}$. Let  $T>0$ and $f\in \mathcal C^{2,\gamma_1}(\mathbf R^{d+1})$. Then,  there exists $C>0$ such that for all $\delta>0$ and $0\leq r<t\leq T$ such that $t-r\leq \delta$, one has for all $N\ge 1$,
$$\mathbf E\big[|\langle f,\mu_t^N\rangle -\langle f,\mu_r^N\rangle |^2\big]\leq C (\delta^2+\delta/N+ 1/N).$$
\end{lemma}

\begin{proof}
Using~\eqref{eq.pre_limitz1zN}, Jensen's inequality,   \eqref{borne_phi-y},  \eqref{eq.E-3}, and \eqref{eq.E-4}, one has for $f\in \mathcal C^{2,\gamma_1}(\mathbf R^{d+1})$,
\begin{align}
\nonumber
\mathbf E\big[|\langle f,\mu_t^N\rangle -\langle f,\mu_r^N\rangle |^2\big]&\leq C\Big[(t-r)^2(1+1/N^2)\|f\|_{\mathcal C^{1,\gamma_1}}^2 +\mathbf E\big[ \big| \sum_{k=\lfloor Nr\rfloor}^{\lfloor Nt\rfloor-1} \mathbf M_k^{N}[f]  \big |^2\big]\\
\label{eq.est-tr}
&\quad +\mathbf E\big[\left| \mathbf W_t^N[f] - \mathbf W_r^N[f] \right|^2\big]+\mathbf E\big[\left| \mathbf R_t^N[f] - \mathbf R_r^N[f] \right|^2\big].
\end{align}
We also have with the same arguments as  those used just before \eqref{bound M_k}
$$\mathbf E\big[ \big| \sum_{k=\lfloor Nr\rfloor}^{\lfloor Nt\rfloor-1} \mathbf M_k^{N}[f]  \big |^2\big]=\sum_{k=\lfloor Nr\rfloor}^{\lfloor Nt\rfloor-1} \mathbf E[|\mathbf M_k^{N}[f]|^2].
$$
Using in addition    \eqref{eq.M1}, one has
 $\mathbf E\big[ \big| \sum_{k=\lfloor Nr\rfloor}^{\lfloor Nt\rfloor-1} \mathbf M_k^{N}[f]  \big |^2\big]\le C (N\delta+1) \|f\|_{\mathcal C^{1,\gamma_1}}^2/ N^2$. Note that with this argument, we also deduce that
 \begin{equation}\label{eq.Mkk}
 \mathbf E[ | \mathbf M_t^N[f]|^2]\le C\|f\|_{\mathcal C^{1,\gamma_1}}^2/ N.
 \end{equation}
  On the other hand, by \eqref{eq.Wt} and \eqref{eq.Rt}, one has
 $$\mathbf E\big[\left| \mathbf W_t^N[f] - \mathbf W_r^N[f] \right|^2\big]\le C   \|f\|^2_{\mathcal C^{1,\gamma_1}}/ N\text{ and } \mathbf E\big[\left| \mathbf R_t^N[f] - \mathbf R_r^N[f] \right|^2\big]\le C   \|f\|_{\mathcal C^{2,\gamma_1}}^2/ N^2.$$
One then plugs all  the previous estimates in~\eqref{eq.est-tr} to deduce  the result  of Lemma \ref{lem_rc_cont_sob}.
\end{proof}

We are now in position to prove Proposition \ref{prop_rc_inP}.
\begin{proof}[Proof of Proposition \ref{prop_rc_inP}]
The proof consists in applying \cite[Theorem 4.6]{jakubowski1986skorokhod} with $E=  \mathcal P_{\gamma_0}(\mathbf R^{d+1})$  and $\mathbb F=\{\mathsf H_f, f\in \mathcal C^\infty_c(\mathbf R^{d+1})\}$ where
$$\mathsf  H_f: \nu \in\mathcal P_{\gamma_0}(\mathbf R^{d+1})\mapsto \langle f, \nu \rangle.$$
The set  $\mathbb F$ on $\mathcal P_{\gamma_0}(\mathbf R^{d+1})$ satisfies Conditions~\cite[(3.1) and (3.2) in Theorem 3.1]{jakubowski1986skorokhod}. Condition (4.8) there follows from Proposition \ref{prop.compact_wasserstein}, Lemma \ref{lem_cc_mu^N}, and Markov's inequality.
Let us now show~\cite[Condition (4.9)]{jakubowski1986skorokhod} is verified, i.e. that for all $f\in \mathcal C^\infty_c(\mathbf R^{d+1})$, the family $(\langle f,\mu^N\rangle)_{N\ge1}$ is relatively compact in $\mathcal D(\mathbf R_+,\mathbf R)$.
 To do this, it suffices to use  Lemma~\ref{lem_rc_cont_sob} and    \cite[Proposition~A.1]{descours2022law} (with $\mathcal H_1=\mathcal H_2=\mathbf R$ there).
 In conclusion, according to \cite[Theorem 4.6]{jakubowski1986skorokhod}, the sequence $(\mu^N)_{N\ge1}\subset \mathcal D(\mathbf R_+,\mathcal P_{\gamma_0}(\mathbf R^{d+1}))$ is relatively compact.
\end{proof}


\subsubsection{Limit points  satisfy the limit equation  \eqref{eq.P2}}


For $f\in \mathcal C^{1,\gamma_0-1}(\mathbf R^{d+1})$
and $t\ge 0$,
we introduce for $\mathsf m\in \mathcal D(\mathbf R_+,\mathcal P_{\gamma_0}(\mathbf R^{d+1}))$,
\begin{align}
\boldsymbol \Phi_t[f]:\mathsf m\mapsto &\Big|\langle f,\mathsf m_t\rangle-\langle f,\mu_0\rangle\nonumber\\
& +\eta\int_{0}^t\int_{\mathsf X\times\mathsf Y}\langle\phi(\cdot,\cdot,x)-y,\mathsf m_s\otimes\gamma\rangle\langle\nabla_\theta f\cdot\nabla_\theta \phi(\cdot,\cdot,x),\mathsf m_s\otimes\gamma\rangle  \pi(\di x,\di y)\di s\nonumber\\
\label{d-lambda2}
&+  \eta\int_0^t \big \langle\nabla_\theta f\cdot \nabla_\theta \mathscr D_{{\rm KL}}(q_{\,_\cdot }^1|P_0^1),\mathsf m_s\big \rangle\di s  \Big|.
\end{align}
Note that $\boldsymbol \Phi_t[f]$ is the function  $\boldsymbol \Lambda_t[f]$ previously defined in \eqref{d-lambda}   for test functions  $f\in\mathcal C^{1,\gamma_0-1}(\mathbf R^{d+1})$ and for $\mathsf m \in \mathcal D(\mathbf R_+,\mathcal P_{\gamma_0}(\mathbf R^{d+1}))$.


\begin{lemma}\label{le.CSobo}
Assume  {\rm \textbf{A1}}$\to$$\mathbf{A5}$. Let $f\in\mathcal C^{1,\gamma_0-1}(\mathbf R^{d+1})$. Then
$\boldsymbol  \Phi_t[f]$ is well defined. In addition, if a sequence $(\mathsf m^N)_{N\ge 1}$ converges to $\mathsf m$ in $\mathcal D(\mathbf R_+,\mathcal P_{\gamma_0}(\mathbf R^{d+1}))$, then,  for all continuity point $t\ge 0$ of $\mathsf m$, we have $\boldsymbol \Phi_t[f](\mathsf m^N)\to \boldsymbol \Phi_t[f](\mathsf m)$.
\end{lemma}

\begin{proof}
Using {\rm \textbf{A1}}, and because $ \mathsf Y$ is bounded and the function $\phi$ is bounded,  $\mathscr  G_1^{x,y}: \theta \mapsto \langle \phi(\theta,\cdot,x)-y,\gamma\rangle \in  \mathcal C^\infty_b(\mathbf{R}^{d+1})$. In addition, for all  multi-index $\alpha \in \mathbf N^{d+1}$, there exists $C>0$, for all $x,y\in \mathsf X \times \mathsf Y$ and all $\theta\in \mathbf R^{d+1}$, $|\partial_\alpha \mathscr G_1^{x,y}(\theta)|\le C$. The same holds for the function
$\mathscr G_2^{x}: \theta\in \mathbf R^{d+1}\mapsto \langle\nabla_\theta \phi(\theta,\cdot,x), \gamma\rangle$.
Consequently, $\theta\mapsto \nabla_\theta f(\theta)\cdot \mathscr G_2^{x}(\theta)\in\mathcal C^{0,\gamma_0-1}(\mathbf R^{d+1})\hookrightarrow \mathcal C^{0,\gamma_0}(\mathbf R^{d+1})$. Then, there exists $C>0$ independent of $(x,y)\in \mathsf X \times \mathsf Y$ and $s\in [0,t]$ such that
$$
|\langle \mathscr G_1^{x,y},\mathsf m_s\rangle|\le  C,$$
 and
$$|\langle \nabla_\theta f\cdot \mathscr G_2^x,\mathsf m_s\rangle |\le C \Vert   f \Vert_{\mathcal C^{1,\gamma_0-1}} \langle 1+|.|^{\gamma_0}, \mathsf m_s\rangle.$$
Finally, the function $\theta \mapsto \nabla_\theta \mathscr D_{{\rm KL}}(q_{\theta }^1|P_0^1)$ is smooth (see~\eqref{eq.kl_1}) and~\eqref{eq.kl_2} extends to all its derivatives, i.e. for all  multi-index $\alpha \in \mathbf N^{d+1}$,
there exists $c>0$, for all $\theta\in \mathbf R^{d+1}$,
$$
|\partial_\alpha \nabla_{\theta}\mathscr D_{{\rm KL}}(q_{\theta}^1|P_0^1)|\le c(1+|\theta|).
$$
Thus, $ \nabla_\theta f\cdot \nabla_\theta \mathscr D_{{\rm KL}}(q_{\theta }^1|P_0^1)\in \mathcal C^{0,\gamma_0}(\mathbf R^{d+1})$ and for some $C>0$ independent of $s\in [0,t]$
$$
|\langle\nabla_\theta f\cdot \nabla_\theta \mathscr D_{{\rm KL}}(q_{\,_\cdot }^1|P_0^1),\mathsf m_s\big \rangle|\le C \Vert   f \Vert_{\mathcal C^{1,\gamma_0-1}} \langle 1+|.|^{\gamma_0}, \mathsf m_s\rangle.
$$
Since in addition $\sup_{s\in [0,t]}\langle 1+|.|^{\gamma_0}, \mathsf m_s\rangle<+\infty$ (since $s\mapsto \langle 1+|.|^{\gamma_0}, \mathsf m_s\rangle \in \mathcal D(\mathbf R_+,\mathbf R$)),  $\boldsymbol  \Phi_t[f]$ is well defined. To prove the continuity property of $\boldsymbol  \Phi_t[f]$ it then suffices to use the previous upper bounds together  similar  arguments as those used in the proof of Lemma \ref{lem_cont_points} (see also \cite{descours2022law}).
\end{proof}

\begin{proposition}\label{prop_conv_le_sob}
Assume  {\rm \textbf{A1}}$\to$$\mathbf{A5}$. Let $\mu^*$ be a limit point of $(\mu^N)_{N\ge1}$ in $\mathcal D(\mathbf R_+,\mathcal P_{\gamma_0}(\mathbf R^{d+1}))$. Then, $\mu^*$ satisfies a.s. Equation  \eqref{eq.P2}.
\end{proposition}

\begin{proof}
Let us consider $f\in \mathcal C_c^\infty(\mathbf R^{d+1})$ and $\mu^*$ be a limit point of $(\mu^N)_{N\ge1}$ in $\mathcal D(\mathbf R_+,\mathcal P_{\gamma_0}(\mathbf R^{d+1}))$. Recall that by~\cite[lemma 7.7 in Chapter 3]{ethier2009markov}, the complementary of the set
$$\mathcal C({\mu^*})=\{t\ge 0, \, \mathbf P(\mu^*_{t^-}= \mu^*_t)=1\}$$
is at most countable. Let $t_*\in \mathcal C({\mu^*})$. Then, by Lemma \ref{le.CSobo}, one has that   $\mathbf P(\mu^*\in\mathsf D(\boldsymbol \Phi_{t_*}[f]))=0$. Thus, by the continuous mapping theorem, it holds
$$\boldsymbol \Phi_{t_*}[f](\mu^N)\xrightarrow{\mathscr D}\boldsymbol \Phi_{t_*}[f](\mu^*).$$
On the other hand, using \eqref{eq.pre_limitz1zN} and the estimates \eqref{eq.Rt}, \eqref{eq.Wt}, \eqref{eq.Mkk}, \eqref{eq.E-3}, and \eqref{eq.E-4}, it holds
\begin{equation*}
\lim_{N\to\infty}\mathbf E[\boldsymbol  \Phi_{t_*}[f](\mu^N)]=0.
\end{equation*}
Consequently, for all $f\in \mathcal C_c^\infty(\mathbf R^{d+1})$ and $t_*\in \mathcal C({\mu^*})$, it holds a.s. $\boldsymbol  \Phi_{t_*}[f](\mu^*)=0$. On the other hand,  for all $\psi\in \mathcal C_c^\infty(\mathbf R^{d+1})$, $\mathsf m\in \mathcal D(\mathbf R_+,\mathcal P_{\gamma_0}(\mathbf R^{d+1}))$, and $s\ge 0$, the mappings
  $$t\ge 0\mapsto \boldsymbol \Phi_{t}[\psi ](\mathsf m)$$ is right continuous, and
$$f\in \mathcal H^{L_0,\gamma_0-1}(\mathbf R^{d+1})\mapsto \boldsymbol \Phi_{s}[f](\mathsf m)$$
 is continuous (because $\mathcal H^{L_0,\gamma_0-1}(\mathbf R^{d+1})\hookrightarrow \mathcal C_0^{1,\gamma_0-1}(\mathbf R^{d+1})$).
In addition, $\mathcal H^{L_0,\gamma_0-1}(\mathbf R^{d+1})$  admits a dense and countable subset of elements in $\mathcal C_c^\infty(\mathbf R^{d+1})$. Moreover, there exists a countable subset $\mathcal T_{\mu^*}$ of
$\mathcal C({\mu^*})$ such that for all $t\ge 0$ and $\epsilon>0$, there exists $s\in \mathcal T_{\mu^*}$,  $s\in [t,t+\epsilon]$.  We prove this claim. Since $\mathbb R_+$ is a metric space, $\mathcal C({\mu^*})$  is   separable  and thus admits a dense subset $\mathcal O_{\mu^*}$. Since $[t+\epsilon/4,t+3\epsilon/4]\cap \mathcal C({\mu^*})\neq \emptyset$, there exists $u\in [t+\epsilon/4,t+3\epsilon/4]\cap \mathcal C({\mu^*})$. Consider now $s\in \mathcal O_{\mu^*}$ such that $|s-u|\le \epsilon/4$. It then holds $t\le s\le t+ \epsilon$, proving the claim with $\mathcal T_{\mu^*}=\mathcal O_{\mu^*}$.

Hence, we have with a classical argument that a.s. for all $f\in \mathcal H^{L_0,\gamma_0-1}(\mathbf R^{d+1})$ and $t\ge 0$,  $\boldsymbol  \Lambda_{t}[f](\mu^*)=0$. Note also that $\mathcal C^\infty_b(\mathbf R^{d+1})\subset \mathcal H^{L_0,\gamma_0-1}(\mathbf R^{d+1})$ since $2\gamma_0>d+1$. This ends the proof of the proposition.
\end{proof}

\subsection{Uniqueness of the limit equation and end of the proof of Theorem \ref{thm.z1zN}}

 In this section, we prove that there is a unique solution to \eqref{eq.P2} in $\mathcal C(\mathbf R_+,\mathcal P_{1}(\mathbf R^{d+1}))$. To this end, we first need to prove that every limit points of $(\mu^N)_{N\ge 1}$ a.s. belongs to $\mathcal C(\mathbf R_+,\mathcal P_{1}(\mathbf R^{d+1}))$.

\subsubsection{Limit points belong to $\mathcal C(\mathbf R_+,\mathcal P_{1}(\mathbf R^{d+1}))$}

\begin{proposition}\label{p-limit in P}
Assume  {\rm \textbf{A1}}$\to$$\mathbf{A5}$. Let $\mu^*\in \mathcal D(\mathbf R_+,\mathcal P_{\gamma_0}(\mathbf R^{d+1}))$ be a limit point of $(\mu^N)_{N\ge 1}$ in $ \mathcal D(\mathbf R_+,\mathcal P_{\gamma_0}(\mathbf R^{d+1}))$. Then, a.s. $\mu^*\in\mathcal C(\mathbf R_+,\mathcal P_{1}(\mathbf R^{d+1}))$.
\end{proposition}

\begin{proof}
 Note that since $\mathsf W_1\le \mathsf W_{\gamma_0}$, $\mu^{N'} \xrightarrow{\mathscr D} \mu^*$ also in $ \mathcal D(\mathbf R_+,\mathcal P_{1}(\mathbf R^{d+1}))$, along some subsequence $N'$. According to \cite[Proposition 3.26 in Chapter VI]{jacod2003skorokhod},  $\mu^*\in\mathcal C(\mathbf R_+,\mathcal P_{1}(\mathbf R^{d+1}))$ a.s.  if  for all $T>0$, $\lim_{N\to +\infty} \mathbf E\big[ \sup_{t\in [0,T]} \mathsf W_1(\mu^N_{t_-},\mu^N_t)  \big]=0$. Using \eqref{Kantorovitch Rubinstein}, this is equivalent to prove that
\begin{equation}\label{eq.L1}
\lim_{N\to +\infty} \mathbf E\Big[ \sup_{t\in [0,T]} \sup_{\Vert f\Vert_{\text{Lip}}\le 1}|\langle f,\mu^N_{t_-}\rangle-\langle f,\mu^N_t\rangle| \Big]=0.
\end{equation}
 Let us consider $T>0$ and a Lipschitz function $f:\mathbf R^{d+1}\to \mathbf R$ such that $\Vert f\Vert_{\text{Lip}}\le 1$. We have $\langle f,\mu_t^N\rangle=\langle f,\mu_0^N\rangle+ \sum_{k=0}^{\lfloor Nt\rfloor-1}\langle f,\nu_{k+1}^N\rangle-\langle f,\nu_k^N\rangle$ (with usual convention $\sum_0^{-1}=0$). Thus the discontinuity points of $t\in [0,T]\mapsto \langle f,\mu_t^N\rangle$ lies exactly at $\{1/N, 2/N,\ldots, \lfloor NT\rfloor/N\}$ and
\begin{align}\label{eq.Bs}
|\langle f,\mu^N_{t_-}\rangle-\langle f,\mu^N_t\rangle|\le \max_{k=0,\ldots,\lfloor NT\rfloor-1}|\langle f,\nu_{k+1}^N\rangle-\langle f,\nu_{k}^N\rangle|, \ \ \forall t\in [0,T], \, f \text{ Lipschitz}.
\end{align}
  Pick $k=0,\ldots,\lfloor NT\rfloor-1$. We have by \eqref{eq.diff},
\begin{align}\label{eq.Bs2}
|\langle f,\nu_{k+1}^N\rangle-\langle f,\nu_{k}^N\rangle|&\le \frac 1N\sum_{i=1}^N  |\theta_{k+1}^i-\theta_k^i| \le \frac CN\sum_{i=1}^N\Big[ \frac{1}{NB}\sum_{\ell=1}^B (1+\mathfrak b(\mathsf Z^{i,\ell}_k)) +\frac{1}{N} (1+|\theta_k^i|)\Big]=:d_k^N
\end{align}
Hence, it holds:
\begin{align*}
|d_k^N|^2 \le   \frac CN\sum_{i=1}^N\Big[ \frac{1}{N^2B}\sum_{\ell=1}^B (1+\mathfrak b^2(\mathsf Z^{i,\ell}_k)) +\frac{1}{N^2} (1+|\theta_k^i|^2)\Big],
\end{align*}
where thanks to Lemma \ref{lem:moment_param} and \textbf{A1}, for all $k=0,\ldots,\lfloor NT\rfloor-1$, $\mathbf E[|d_k^N|^2]\le C/N^2$ for some  $C>0$ independent of $N\ge 1$ and $k=0,\ldots,\lfloor NT\rfloor-1$.
Thus, using \eqref{eq.Bs} and \eqref{eq.Bs2},
\begin{align*}
\mathbf E\Big[ \sup_{t\in [0,T]} \sup_{\Vert f\Vert_{\text{Lip}}\le 1}|\langle f,\mu^N_{t_-}\rangle-\langle f,\mu^N_t\rangle| \Big]&\le \mathbf E\Big[ \sup_{\Vert f\Vert_{\text{Lip}}\le 1} \max_{k=0,\ldots,\lfloor NT\rfloor-1}|\langle f,\nu_{k+1}^N\rangle-\langle f,\nu_{k}^N\rangle| \Big]\\
&\le \mathbf E\Big[  \max_{k=0,\ldots,\lfloor NT\rfloor-1}d_k^N \Big] \\
&\le \mathbf E\Big[  \sqrt{\sum_{k=0}^{\lfloor NT\rfloor-1} |d_k^N|^2 }\Big]\\
&\le \sqrt{\mathbf E\Big[   \sum_{k=0}^{\lfloor NT\rfloor-1} |d_k^N|^2 \Big]}\le \frac C{\sqrt N}.
\end{align*}
This concludes the proof of Proposition \ref{p-limit in P}.
\end{proof}

\subsubsection{Uniqueness of the solution to \eqref{eq.P2}}

\begin{proposition}\label{pr.uni2}
There is  a  unique solution $\bar\mu\in\mathcal C(\mathbf R_+,\mathcal P_1(\mathbf R^{d+1}))$ to \eqref{eq.P2}.
\end{proposition}


\begin{proof}
First of all, the existence of a solution is provided by  Propositions~\ref{prop_rc_inP},~\ref{p-limit in P} and~\ref{prop_conv_le_sob}.
Let us now prove that there is a unique solution to  \eqref{eq.P2}   in $\mathcal C(\mathbf R_+,\mathcal P_{1}(\mathbf R^{d+1}))$.

Recall the definition of   $ \boldsymbol v[\mu]$   in \eqref{def V[mu]}.  We claim that
 for all $T>0$ and
 all solution $\bar\mu\in\mathcal C(\mathbf R_+,\mathcal P_{1}(\mathbf R^{d+1}))$ of \eqref{eq.P2}, there exists $C>0$ such that
\begin{equation}\label{eq.itemB}
| \boldsymbol v[\bar\mu_t](\theta)- \boldsymbol v[\bar\mu_s](\theta)|\leq C|t-s|, \ \text{ for all  $ 0\leq s \le  t\leq T$   and  $\theta\in\mathbf R^{d+1}$}.
\end{equation}
 The proof of item \eqref{eq.itemB} is the same as the one made for Item 2 in Proposition~\ref{p-uniq} since it holds using \eqref{eq.kl_2} and \eqref{eq.Bntheta}, for all $0\le s\le  t\le T$ and $z\in \mathbf R^d$,
\begin{align*}
\Big |\int_s^t\langle\nabla_\theta\phi(\cdot,z,x)\cdot \nabla_\theta \mathscr D_{{\rm KL}}(q_{\,_\cdot }^1|P_0^1),\bar\mu_r\rangle\di r\Big |&\le C\mathfrak b(z)\int_s^t \langle (1+|\cdot| ), \bar \mu_r\rangle \di r\\
&\le  C\mathfrak b(z)\, \max_{r\in [0,T]}\langle (1+|\cdot| ), \bar \mu_r\rangle |t-s|.
\end{align*}

We now  conclude the proof of   Proposition \ref{pr.uni2}.
Item 1 in the proof of Proposition~\ref{p-uniq}  and \eqref{eq.itemB}    imply that $v(t,\theta)= \boldsymbol v[\bar\mu_t](\theta)$ is globally Lipschitz on  $[0,T]\times  \mathbf R^{d+1}$, for all $T>0$, when $\bar\mu\in\mathcal C(\mathbf R_+,\mathcal P_{1}(\mathbf R^{d+1}))$ is a solution   of \eqref{eq.P2}. Since  in addition   a solution $\bar \mu$  to \eqref{eq.P2}  is a weak solution on $\mathbf R_+$ to \eqref{eq.measure} in $\mathcal C(\mathbf R_+,\mathcal P(\mathbf R^{d+1}))$, it holds by~\cite[Theorem 5.34]{villani2021topics}:
\begin{equation}\label{eq.vi2}
\forall t\ge 0, \ \bar \mu_t=\phi_t\# \mu_0,
\end{equation}
 where $\phi_t$ is the flow generated by the vector field $\boldsymbol v[\bar \mu_t](\theta)$ over $\mathbf R^{d+1}$.
  Together with Item 3 in the proof of  Proposition \ref{p-uniq}  and using  the same arguments as those used in Step~3 of the proof of~\cite[Proposition 2.14]{descours2022law}, two solutions  agrees on each $[0,T]$ for all $T>0$. One then deduces the uniqueness of the solution to \eqref{eq_limit}. The proof of Proposition \ref{pr.uni2} is complete.
\end{proof}

We are now in position to end the proof of Theorem \ref{thm.z1zN}.

\begin{proof}[Proof of Theorem \ref{thm.z1zN}]
By Proposition \ref{prop_rc_inP},  $(\mu^N)_{N\ge1}$ is relatively compact in $\mathcal D(\mathbf R_+,\mathcal P_{\gamma_0}(\mathbf R^{d+1}))$.  Let $\mu^1,\mu^2\in \mathcal D(\mathbf R_+,\mathcal P_{\gamma_0}(\mathbf R^{d+1}))$ be two limit points of this sequence. By Proposition \ref{p-limit in P},  a.s. $\bar\mu^1,\bar\mu^2\in \mathcal C(\mathbf R_+,\mathcal P_{1}(\mathbf R^{d+1}))$. In addition, according to  Proposition \ref{prop_conv_le_sob}, $\mu^1$ and $ \mu^2$ are a.s.  solutions of \eqref{eq.P2}.   Denoting by $\bar\mu\in  \mathcal C(\mathbf R_+,\mathcal P_{\gamma_0}(\mathbf R^{d+1}))$ the unique solution to \eqref{eq.P2} (see Proposition \ref{pr.uni2}), we have a.s.
$$\bar \mu^1 =\bar\mu  \text{ and } \bar \mu^2=\bar\mu \text{ in } \mathcal C(\mathbf R_+,\mathcal P_{1}(\mathbf R^{d+1})).$$
In particular  $\bar\mu\in \mathcal D(\mathbf R_+,\mathcal P_{\gamma_0}(\mathbf R^{d+1}) )$ and  $\bar\mu^j=\bar\mu$   in $\mathcal D(\mathbf R_+,\mathcal P_{\gamma_0}(\mathbf R^{d+1}))$, $j\in \{1,2\}$. As a consequence, $\bar\mu$ is the unique limit point of $(\mu^N)_{N\ge1}$   in $\mathcal D(\mathbf R_+,\mathcal P_{\gamma_0}(\mathbf R^{d+1}))$ and the whole sequence $(\mu^N)_{N\ge1}$  converges to  $\bar\mu$ in    $\mathcal D(\mathbf R_+,\mathcal P_{\gamma_0}(\mathbf R^{d+1}))$. Since $\bar\mu$ is deterministic, the  convergence also holds in probability. The proof of Theorem \ref{thm.z1zN} is complete.
\end{proof}

Let us now prove Proposition \ref{pr.u}.
\begin{sloppypar}
\begin{proof}[Proof of Proposition \ref{pr.u}]
Any solution to  \eqref{eq_limit} in $\mathcal C([0,T],\mathcal P(\Theta_T))$ is a solution to  \eqref{eq.P2} in $\mathcal C([0,T],\mathcal P_{1}( \mathbf R^{d+1}))$. The result follows from  Proposition \ref{pr.uni2}.
\end{proof}
\end{sloppypar}

\end{document}